\documentclass[review]{elsarticle}

\usepackage{amssymb}

\usepackage{graphicx}
\usepackage{amsthm}
\usepackage{latexsym,multicol,caption}     
\usepackage{algorithm}  
\usepackage{algorithmic}
\usepackage{epsfig} 
\usepackage{epstopdf}
\usepackage[colorlinks, linkcolor=blue, anchorcolor=red, citecolor=red, urlcolor=black]{hyperref}
\usepackage{color}
\usepackage{multirow}
\usepackage{multicol}
\usepackage{subfigure}
\usepackage{placeins}
\usepackage{bm}
\usepackage{graphics}
\usepackage{footnote}
\usepackage{url}
\usepackage{indentfirst}
\usepackage{longtable}
\usepackage{array}
\usepackage{mdwmath}
\usepackage{mdwtab}
\usepackage{rotating}
\usepackage{color}
\usepackage{morefloats}
\usepackage{slashbox}
\usepackage{eqnarray}
\usepackage{natbib}
\usepackage{float}
\usepackage{makecell}
\usepackage{amsmath}
\usepackage{threeparttable}
\usepackage{pifont}
\usepackage{cleveref}
\usepackage{balance}
\usepackage{flushend}

\newcommand{\argmin}{\arg\!\min}

\crefname{equation}{}{}
\crefname{figure}{Figure}{Figures}
\crefname{table}{Table}{Tables}
\crefname{section}{Section}{Sections}
\crefname{theorem}{Theorem}{Theorems}
\crefname{lemma}{Lemma}{Lemmas}

\theoremstyle{plain}
\newtheorem{theorem}{Theorem}[]
\newtheorem{lemma}{Lemma}[]

\newtheorem{definition}{Definition}[]

\theoremstyle{definition}

\theoremstyle{remark}
\newtheorem*{rem}{Remark}

\newcommand{\tr}[1]{\textbf{Tr}{(#1)}}

\newcommand{\dd}[1]{{\text{d}}{#1}}

\journal{Knowledge-Based Systems}

\begin{document}
	
	\begin{frontmatter}
		
		\title{ Log-based Sparse Nonnegative Matrix Factorization for Data Representation }

		\author[label1]{Chong Peng}
		\author[label1]{Yiqun Zhang}
		\author[label2]{Yongyong Chen}
		\author[label3]{Zhao Kang}
		\author[label1]{Chenglizhao Chen\corref{mycorresponding}}
		\ead{cclz123@163.com}
		\author[label4,label5]{Qiang Cheng}
		\cortext[mycorresponding]{Corresponding author: Chenglizhao Chen}

		\address[label1]{College of Computer Science and Technology, Qingdao University}
		\address[label2]{Department of Computer Science, Harbin Institute of Technology}
		\address[label3]{School of Computer Science and Engineering, University of Electronic Science and Technology of China}
		\address[label4]{Department of Computer Science, University of Kentucky}
		\address[label5]{Institute of Biomedical Informatics, University of Kentucky}
		
		\begin{abstract}
			Nonnegative matrix factorization (NMF) has been widely studied in recent years due to its effectiveness in representing nonnegative data with parts-based representations. 
			For NMF, a sparser solution implies better parts-based representation.
			However, current NMF methods do not always generate sparse solutions.
			In this paper, we propose a new NMF method with log-norm imposed on the factor matrices to enhance the sparseness.
			Moreover, we propose a novel column-wisely sparse norm, named $\ell_{2,\log}$-(pseudo) norm to enhance the robustness of the proposed method.
			The $\ell_{2,\log}$-(pseudo) norm is invariant, continuous, and differentiable.
			For the $\ell_{2,\log}$ regularized shrinkage problem, we derive a closed-form solution, which can be used for other general problems.
			Efficient multiplicative updating rules are developed for the optimization, which theoretically guarantees the convergence of the objective value sequence.
			Extensive experimental results confirm the effectiveness of the proposed method, as well as the enhanced sparseness and robustness.
		\end{abstract}
		
		
	\end{frontmatter}
	
	
	\section{Introduction}

	It has been increasingly ubiquitous to use high-dimensional data in various areas such as machine learning, data mining, and multimedia data processing, 
	which makes the task of learning from examples challenging \cite{duda2012pattern,peng2020discriminative}. 
	A widely used technique to handle such data is dimension reduction, among which matrix factorization methods have drawn significant attention.
	Matrix factorization seeks two or more low-dimensional matrices to approximate the original data such that the high-dimensional data can be represented with reduced dimensions \cite{liu2013robust,peng2021kernel}. 
	Typical matrix factorization techniques include principal component analysis (PCA), nonnegative matrix factorization (NMF), singular value decomposition (SVD), eigenvalue decomposition (EVD), etc.
	
	For some types of data, the entries are naturally nonnegative. 
	For example, the pixel values of images or the frequencies of words in a document are naturally nonnegative. 
	For such data, parts-based representation is believed to commonly exist in human brain with both psychological and physiological evidence \cite{palmer1977hierarchical,wachsmuth1994recognition,logothetis1996visual}. 
	The parts-based representation inspires us to seek two nonnegative factor matrices to approximate the original nonnegative data, which leads to the NMF. 
	Among the two factor matrices, one is considered as the basis while another is treated as the representation or soft indicator matrix.
	NMF only allows additive combinations of the basis vectors, which enables NMF to learn a parts-based representation \cite{lee1999learning}.
	
	NMF has been extensively studied in recent years \cite{Yifeng2019Multiclass,Wang2019Multiview,Li2018Robust,Yin2018Scalable,peng2022two,schmidt2009bayesian,dalhoumi2021bayesian}, 
	which has found applications in various areas, such as pattern recognition \cite{li2001learning}, multimedia analysis \cite{cooper2002summarizing}, and text mining \cite{xu2003document}.
	Basically, the goal of the original NMF is to approximate a nonnegative matrix with two nonnegative factor matrices with physical meanings \cite{lee1999learning}. 
	For the constrained optimization problem, multiplicative updating-based strategy has been developed and commonly adopted for its optimization \cite{lee2001algorithms}. 
	The original NMF seeks the factorization in Euclidean space and thus it only accounts for the linear relationship of data while omitting the nonlinear ones.
	To tackle this issue, variants such as the graph regularized NMF (GNMF) \cite{cai2011graph}, the robust manifold NMF (RMNMF) \cite{huang2014robust} are developed based on manifold learning.
	Other than manifold technique, kernel method is also used in NMF methods, which relies on the convex NMF framework \cite{ding2010convex,peng2020nonnegative}.
	It assumes that the basis can be represented as a combination of the data and thus it is doable to calculate the similarity of examples and basis in kernel space.

	For NMF methods, it is pointed out that sparser solutions reveal better parts-based representation \cite{cai2011graph}.
	However, recent studies show that NMF does not always generate sparse factorization, which implies failure in learning parts-based representation \cite{hoyer2004non,hoyer2002non}. 
	{To combat this issue, various approach have been attempted for sparse solutions, such as Bayesian sparse learning \cite{chien2012bayesian,fedorov2018unified}.
		For example, a maximum a posteriori (MAP) estimation framework is developed to address the sparse nonnegative matrix factorization problems \cite{fedorov2018unified},
		which is built upon the sparse Bayesian learning.
		The sparse Bayesian learning framework places a sparsity-promoting prior on the data \cite{wipf2004sparse}, 
		which has been shown to give rise to many models in literature \cite{wipf2010iterative}.
		Bayesian group sparse learning introduces Laplacian scale mixture distribution for sparse coding given a sparseness control parameter \cite{chien2012bayesian}.}
	It is natural to impose the $\ell_0$-norm for sparseness, which counts the number of nonzero elements in a matrix.
	However, $\ell_0$-norm is generally hard to solve and the $\ell_1$-norm has been widely used as a relaxation for sparseness property in the machine learning community. 
	Unfortunately, the $\ell_1$ may be inaccurate to approximate $\ell_0$ if there are large entries in the input matrix or vector. 
	Recently, nonconvex approximations to the rank function have drawn significant attention in various applications such as subspace learning \cite{peng2015subspace} and robust PCA \cite{Peng2019Robust}.
	It is shown that nonconvex approximations such as log-determinant rank approximation have been successful in low-rank matrix recovery problems.
	The reason is that the nonconvex approximations can better reveal the behavior of the true rank function than the convex approach.
	In fact, the low-rankness of a matrix is closely related to the sparsity of its singular values, where the rank function is equivalent to the $\ell_0$-norm of the vector of singular values.
	Thus, the success of nonconvex approximations to the rank function inspires us to design nonconvex approximations to the $\ell_0$-norm for enhanced sparse property.

	Moreover, nonnegative data such as images often have noise, which has adverse effects and degrades the learning performance.
	Thus, there is a crucial need to tackle noise effects from data \cite{huang2014robust,peng2017robust}. 
	The $\ell_{2,1}$-norm, which is defined as the summation of $\ell_2$-norms of all column vectors in a matrix, has been widely used to enforce example-wise sparsity to deal with noise effects \cite{nie2010efficient,huang2014robust,xu2010robust,mccoy2011two}. 
	Different from $\ell_1$-norm that treats all entries independently, $\ell_{2,1}$-norm measures input matrix in an example-wise way, which allows it to preserve spatial information of examples \cite{huang2014robust}. 
	For nonnegative data such as images, such property of $\ell_{2,1}$-norm is indeed desirable.
	It is noted that the calculation of $\ell_{2,1}$-norm is closely related to the $\ell_1$-norm in that it adds the $\ell_2$-norms of all columns with equal weights.
	Thus, the $\ell_{2,1}$ may have similar issues to $\ell_1$-norm in resulting column-wise sparse property.
	In this paper, to better tackle noise issues and enhance column-wise sparsity, we propose a novel $\ell_{2,\log}$-(pseudo) norm with column-wisely sparser property. 
	For the $\ell_{2,\log}$-based shrinkage problem, we provide a closed-form solution, which can be generally used in various other problems.
	
	We summarize the key contributions of this paper as follows:
	1) We propose a novel NMF model with log-based sparsity constraints. The new model generates sparser solutions to the factorization, which reveals better parts-based representation;
	2) Multiplicative updating rules are developed for efficient optimization;
	3) Regarding the updating rules, we provide theoretical analysis that guarantees the convergence of our algorithm;
	4) We propose a novel $\ell_{2,\log}$-(pseudo) norm to restrict column-wise sparsity. The $\ell_{2,\log}$-(pseudo) norm is invariant.
	Similar to the soft-thresholding problem, we formally provide the $\ell_{2,\log}$-shrinkage operator,
	which is the solution to the $\ell_{2,\log}$-(pseudo) norm associated thresholding problem.
	The $\ell_{2,\log}$-shrinkage operator can be generally used in other problems;
	5) The $\ell_{2,\log}$-shrinkage operator guarantees that the data with noise subtraction are nonnegative at each iteration,
	which ensures the nonnegativity and convergence of the factorization;
	6) Extensive experiments confirm the effectiveness of our method in clustering and data representation. 
	

	The rest of this paper is organized as follows: 
	In \cref{sec_related}, we briefly review some methods that are closely related to our research. 
	Then we introduce our method, including its formulation, optimization, and convergence analysis in \cref{sec_proposed}. 
	To enhance the robustness of the new method to noise effects, we present the robust model in \cref{sec_robust}, including its formulation, optimization, and convergence analysis.
	We conduct extensive experiments and present the detailed results in \cref{sec_exp}. 
	Finally, we conclude the paper in \cref{sec_conclusion}.
	
	{\textbf{Remark}}: In this paper, the proposed log-based sparse approximations do not satisfy the definition of norms. They are named as (pseudo) norms for simplicity of representation.


	\section{Related Work}
	\label{sec_related}
	In this section, we briefly review some closely related works, including the original NMF and graph Laplacian.
	\subsection{Original NMF} 
	Given nonnegative data $X = [x_1,\cdots,x_n]\in\mathcal{R}^{p\times n}$ with $p$ being the dimension and $n$ sample size, NMF is to factor $X$ into $U\in\mathcal{R}^{p\times k}$ (basis) and $V\in\mathcal{R}^{n\times k}$ (coefficients) with the following optimization problem \cite{lee1999learning}:
	\begin{equation}
	\label{eq_nmf}
	\min_{u_{ij}\ge 0,v_{ij}\ge 0} \|X - UV^T \|_F^2,
	\end{equation}
	where $u_{ij}$ and $v_{ij}$ are the $ij$-th elements of $U$ and $V$, respectively, and $k\ll n$ enforces a low-rank approximation of the original data. 
	
	\subsection{Graph Laplacian} 
	To account for nonlinear relationships of data in a mapped low-dimensional space, graph Laplacian is a powerful technique that has been widely used \cite{chung1997spectral}.
	It is defined as  
	\begin{equation}
	\label{eq_manifold}
	\begin{aligned}
	\textbf{Tr}(V^T D V) - \textbf{Tr}(V^T W V)= \textbf{Tr}(V^T L V),
	\end{aligned}
	\end{equation}
	where $W=[w_{ij}]$ is a similarity matrix that measures pair-wise nonlinear similarities of the examples, 
	$D=[d_{ij}]$ is a diagonal matrix with $d_{ii} = \sum_{j}w_{ij}$,
	and $L=D-W$ is the graph Laplacian matrix.

	\section{Log-norm Regularized Sparse NMF}
	\label{sec_proposed}
	In this section, we will present the log-norm regularized sparse NMF model, including its formulation, optimization, and convergence analysis.
	
	\subsection{Formulation}
	For nonnegative data such as images, the parts-based representation has been shown effective in representing their latent structures. 
	This inspires us to develop parts-based representation to represent such data. 
	The NMF methods have been extensively studied for parts-based representation, 
	which are assumed to generate sparse representation due to the fact that they only involve positive combination of the basis.
	However, recent studies show that NMF does not always result in sparse factorization \cite{hoyer2004non,hoyer2002non},
	which implies that NMF does not always succeed in finding good parts-based representation.
	It is crucial for NMF methods to have sparse solutions, since sparser solutions imply better parts-based representation \cite{cai2011graph}, which reveals the underlying true structures of the data.
	{To enhance the sparseness and reveal the nature of parts-based representation, 
		various approaches have been developed in the literature, such as Bayesian NMF \cite{schmidt2009bayesian} and regularization technique \cite{hoyer2004non}.}
	In this paper, we adopt the most widely used approach and impose the sparsity constraints with the $\ell_1$-norm regularization on the basis and representation matrices $U$ and $V$, respectively, which leads to 
	\begin{equation}
	\label{eq_obj_l1}
	\begin{aligned}
	\min_{U,V} &  \|X-UV^T\|_F^2 + \alpha \|U\|_{1} + \beta \|V\|_{1}\\
	s.t.  & \quad u_{ij}\ge 0, v_{ij} \ge 0,
	\end{aligned}
	\end{equation}
	where $\alpha\ge0$ and $\beta\ge0$ are two balancing parameters, and $\|M\|_1=\sum_{ij}|m_{ij}|$ is the $\ell_1$-norm for a matrix $M=[m_{ij}]$. 
	In practice, the $\ell_1$-norm is often used as a tight convex relaxation of the $\ell_0$-norm, whereas the latter is rarely used in practice due to the hard optimization. 
	{
		However, the $\ell_1$ sparsity measure becomes an increasingly bad proxy to the $\ell_0$ norm if any of the elements are large \cite{candes2008enhancing}.
		When the input matrix has large values, the $\ell_1$-norm may approximate the $\ell_0$-norm with significant error, 
		which may lead to inaccurate approximation and sub-optimal solution.
		Although greater values of $\alpha$ and $\beta$ may lead to sparser solutions,
		such solutions do not necessarily have good interpretations for the data.
		To obtain sparsity with a learning model, it is natural to require that the model first approximates the sparsity accurately.
		Unfortunately, the $\ell_1$-norm may be inaccurate in approximating sparsity and more accurate approximation is admirable for sparsity learning.}
	To withdraw this drawback and better restrict the sparsity, we propose to impose the following regularization instead of the $\ell_1$-norm:
	$$\|M\|_{\log} = \sum_{ij} \log(1+|m_{ij}|).$$
	In this paper, we call the above term log-norm ($\ell_{\log}$-(pseudo) norm), which can be considered as a special case in \cite{candes2008enhancing}. Particularly, the $\ell_{\log}$-(pseudo) norm admits the followings properties:
	\begin{itemize}
		\item[$\bullet$] $\|M\|_{\log}\ge0$ always holds.
		\item[$\bullet$] For large $|m_{ij}|$, we have $\log(1+|m_{ij}|)\ll|m_{ij}|$. 
		This reveals that $\|M\|_{\log}\ll\|M\|_1$, implying closer approximation to the true sparsity if a matrix contains large values.
	\end{itemize}
	
	We replace the $\ell_1$-norm with the $\ell_{\log}$-(pseudo) norm in \cref{eq_obj_l1} and obtain the following model:
	\begin{equation}
	\label{eq_obj_log2}
	\begin{aligned}
	\min_{U,V} &  \|X-UV^T\|_F^2 + \alpha \|U\|_{\log} + \beta \|V\|_{\log}\\
	s.t. & \quad u_{ij}\ge 0, v_{ij} \ge 0.
	\end{aligned}
	\end{equation}
	It is seen that the $\ell_{\log}$-(pseudo) norm imposed on the basis and representation matrices renders the model to restrict more accurate sparse constraints, which may lead to sparser solutions. 
	For NMF methods, multiplicative optimization strategies are often adopted.
	With the $\ell_{\log}$-(pseudo) norm regularization, it is seen that it is more challenging to design efficient multiplicative optimization algorithm for \cref{eq_obj_log2}.
	It is noted that model \cref{eq_obj_log2} only considers the linear relationships of data in Euclidean space.
	However, this might be less sufficient since there often exist nonlinear relationships of data, which are omitted in the above model.
	To account for nonlinear structures of the data, we extend the above model and seek the representation on manifold, which leads to the log-norm regularized sparse NMF  (LS-NMF):
	%
	\begin{equation}
	\label{eq_obj_log}
	\begin{aligned}
	\min_{U,V} & \|X-UV^T\|_F^2 + \lambda \textbf{Tr}(V^TLV)   + \alpha \|U\|_{\log} + \beta \|V\|_{\log}\\
	s.t. & \quad u_{ij}\ge 0, v_{ij} \ge 0,
	\end{aligned}
	\end{equation}
	where $\lambda\ge 0$ is a balancing parameter and $L$ is the graph Laplacian matrix.
	It is seen that with the graph Laplacian, the new model \cref{eq_obj_log} enforces the smoothness of the data representation from nonlinear kernel space to the low-dimensional space,
	such that the learned representation $V$ represents nonlinear structures of the data.
	In rest of this section, we will develop an efficient multiplicative updating rule for the optimization of \cref{eq_obj_log} and provide the corresponding convergence analysis.
	We will further extend \cref{eq_obj_log} to its robust version in the next section.

	\subsection{Optimization}
	In this subsection, we will design an efficient optimization algorithm for \cref{eq_obj_log} and present the detailed derivations.
	The objective in \cref{eq_obj_log} is equivalent to
	\begin{equation}
	\begin{aligned}
	\mathcal{O} = & \tr{XX^T} - 2\tr{XVU^T} + \tr{UV^TVU^T} \\
	&+ \lambda \tr{V^TLV} + \alpha \|U\|_{\log} + \beta \|V\|_{\log}.
	\end{aligned}
	\end{equation}
	The Lagrangian function of $\mathcal{O}$ is
	\begin{equation}
	\begin{aligned}
	\mathcal{L} = & \tr{XX^T} - 2\tr{XVU^T} + \tr{UV^TVU^T} \\
	&+ \lambda \tr{V^TLV} + \alpha \|U\|_{\log} + \beta \|V\|_{\log} + \tr{\Psi U^T} + \tr{\Phi V^T},
	\end{aligned}
	\end{equation}
	where $\Psi=[\psi_{ij}]$ and $\Phi=[\phi_{ij}]$ are two matrices with $\psi_{ij}$ and $\phi_{ij}$ being the Lagrangian multipliers of constraints $u_{ij}\ge 0$ and $v_{ij}\ge 0$, respectively.
	We take the partial derivatives w.r.t $U$ and $V$, respectively, and obtain
	\begin{equation}
	\begin{aligned}
	\frac{\partial{\mathcal{L}}}{\partial{U}} &=  -2XV + 2UV^TV + \alpha\textbf{1}_{d\times k}\oslash(\textbf{1}_{d\times k}+U) + \Psi\\
	\frac{\partial{\mathcal{L}}}{\partial{V}} & = -2X^TU + 2VU^TU + 2\lambda LV  + \beta\textbf{1}_{n\times k}\oslash(\textbf{1}_{n\times k}+V) + \Phi
	\end{aligned}
	\end{equation}
	where $\oslash$ is the element-wise division operation of two matrices, $\textbf{1}$ is matrix of 1's with size being clarified in the corresponding subscript.
	Using the Karush-Kuhn-Tucker (KKT) conditions of $\psi_{ij}u_{ij} = 0$ and $\phi_{ij}v_{ij} = 0$, we have
	\begin{equation}
	\label{eq_deriv_u}
	\begin{aligned}
	&  -2(XV)_{ij}u_{ij} + 2(UV^TV)_{ij}u_{ij}  + \alpha (\textbf{1}_{d\times k}\oslash(\textbf{1}_{d\times k}+U))_{ij}u_{ij} + \psi_{ij}u_{ij} \\
	= &  -2(XV)_{ij}u_{ij} + 2(UV^TV)_{ij}u_{ij}  + \alpha\frac{u_{ij}}{1+u_{ij}} =0,
	\end{aligned}
	\end{equation}
	\begin{equation}
	\label{eq_deriv_v}
	\begin{aligned}
	& -2(X^TU)_{ij}v_{ij} + 2(VU^TU)_{ij}v_{ij} + 2\lambda (LV)_{ij}v_{ij} 
	+ \beta(\textbf{1}_{n\times k}\oslash(\textbf{1}_{n\times k}+V))_{ij}v_{ij} + \phi_{ij}v_{ij} \\
	= & -2(X^TU)_{ij}v_{ij} + 2(VU^TU)_{ij}v_{ij} + 2\lambda (DV)_{ij}v_{ij} - 2\lambda (WV)_{ij}v_{ij}+ \beta\frac{v_{ij}}{1+v_{ij}}  =0.
	\end{aligned}
	\end{equation}
	The above equations \cref{eq_deriv_u,eq_deriv_v} lead to the following updating rules of the proposed model:
	\begin{align}
	u_{ij} & \leftarrow u_{ij} \frac{(XV)_{ij}}{(UV^TV)_{ij}  + \alpha\frac{1}{2(1+u_{ij})} } \\
	v_{ij} & \leftarrow v_{ij} \frac{(X^TU)_{ij}+ \lambda (WV)_{ij}}{(VU^TU)_{ij} + \lambda (DV)_{ij} + \beta\frac{1}{2(1+v_{ij})}},
	\end{align}
	or equivalently
	\begin{align}
	\label{eq_update_u}
	u_{ij} & \leftarrow u_{ij} \frac{(2XV)_{ij}}{(2UV^TV  + \alpha\textbf{1}_{d\times k}\oslash(\textbf{1}_{d\times k}+U) )_{ij} } \\
	\label{eq_update_v}
	v_{ij} & \leftarrow \! v_{ij}\! \frac{ 2(X^TU + \lambda WV)_{ij}}
	{(2VU^TU \!+\! 2\!\lambda \!DV \!+\! \beta \textbf{1}_{n\times k} \!\oslash\! (\textbf{1}_{n\times k} \!+\! V))_{ij}}.
	\end{align}
	For the convergence of the above updating rules, we will provide the detailed theoretical analysis in the following subsection.
	It is noted that the detailed derivations of optimization and analysis of convergence for \cref{eq_obj_log} are crucial because they are necessary for the robust model in the next section. 
	
	

	\subsection{Convergence Analysis}
	\label{sec_conv_sl}
	In this section, we theoretically analyze the convergence of the updating rules provided in \cref{eq_update_u,eq_update_v}.
	Regarding the updating rules, we have the following theorem.
	\begin{theorem}
		\label{thm_conv}
		The objective $\mathcal{O}$ in \cref{eq_obj_log} is non-increasing under the updates in \cref{eq_update_u,eq_update_v}.
		The objective function is invariant under these updates if and only if $U$ and $V$ are at a stationary point.
	\end{theorem}
	
	In the following, we will provide the proof regarding the updates of $U$ and $V$, respectively.
	To begin the proof, we need to introduce the definition of auxiliary function, which is described below.
	\begin{definition}
		\label{def_aux}
		For the functions $G(v,v')$ and $F(v)$, if the following conditions
		$$G(v,v')\ge F(v),\quad G(v,v) = F(v)$$ are satisfied, then $G(v,v')$ is an auxiliary function of $F(v)$.
	\end{definition}
	Regarding auxiliary function, it has a useful property, which is provided in the following lemma.
	\begin{lemma}[\cite{cai2011graph}]
		\label{lemma_aux}
		If $G(v,v')$ is an auxiliary function of $F(v)$, then $F(v)$ is non-increasing under the updating rule of
		\begin{equation}
		\label{eq_aux_min}
		v^{(t+1)} = \argmin_{v} G(v,v^{(t)}).
		\end{equation}
	\end{lemma}
	\begin{proof}
		The statement is easily seen according to the following chain of inequality:
		\begin{equation}
		F(v^{(t+1)}) \le G(v^{(t+1)},v^{(t)}) \le G(v^{(t)},v^{(t)}) = F(v^{(t)}).
		\end{equation}
	\end{proof}
	It is seen that \cref{lemma_aux} guarantees that the objective function value of $F(v)$ non-increasing if the updating rule of \cref{eq_aux_min} is adopted.
	Next, we will show that the updates of \cref{eq_update_u,eq_update_v} are exactly the updating rule of \cref{eq_aux_min} with proper auxiliary functions of $U$ and $V$, respectively.
	
	Since the updating rules are essentially performed in element-wise manner,
	it is sufficient to show that the objective function is non-increasing with respect to each element of the matrix variable.
	In the following, we first consider the updating of $V$. 
	For $V$, we denote each of its elements by $v_{ab}$.
	Correspondingly, we use $F_{ab}$ to denote the $v_{ab}$-associated part in $\mathcal{O}$.
	Then it is straightforward to obtain the first and second derivatives of $F_{ab}$ with respect to $v_{ab}$ as follows:
	\begin{align}
	F'_{ab} = & (-2X^TU + 2VU^TU + 2\lambda LV)_{ab} + \frac{\beta}{1+v_{ab}}, \\
	F''_{ab} = & 2(U^TU)_{bb} + 2\lambda L_{aa} - \frac{\beta}{(1+v_{ab})^2}.
	\end{align}
	Then, we formally have the following lemma, which defines an auxiliary function for $F_{ab}$.
	
	\begin{lemma}
		\label{lemma_proof_v_G}
		The function
		\begin{equation}
		\label{eq_proof_v_G}
		\begin{aligned}
		G(v,v_{ab}^{(t)}) = & F_{ab}(v_{ab}^{(t)}) + F'_{ab}(v_{ab}^{(t)})(v-v_{ab}^{(t)}) \\
		&  + \frac{(VU^TU)_{ab} + \lambda (DV)_{ab} + \frac{\beta}{2(1+v_{ab}^{(t)})} }{v_{ab}^{(t)}} (v-v_{ab}^{(t)})^2
		\end{aligned}
		\end{equation}
		is an auxiliary function for $F_{ab}(v)$.
	\end{lemma}
	\begin{proof}
		It is easy to check the second condition for auxiliary function, which is seen as below:
		\begin{equation}
		\begin{aligned}
		G(v,v) =&  F_{ab}(v) + F'_{ab}(v)(v-v) \\
		& + \frac{(VU^TU + \lambda DV)_{ab} + \frac{\beta}{2(1+v)}}{v} (v-v)^2 \\
		= & F_{ab}(v).
		\end{aligned}
		\end{equation}
		To show that the first condition, i.e., $G(v,v_{ab}^{(t)}) \ge F_{ab}(v)$ holds, we compare $G(v,v_{ab}^{(t)})$ with the Tylor expansion series of $F_{ab}(v)$, 
		which is expanded as follows:
		\begin{equation}
		\label{eq_proof_v_talor}
		\begin{aligned}
		F_{ab}(v)  = & F_{ab}(v_{ab}^{(t)}) + F'_{ab}(v_{ab}^{(t)})(v-v_{ab}^{(t)}) \\
		&  + \left[(U^TU)_{bb} + \lambda L_{aa} - \frac{\beta}{2(1+v_{ab}^{(t)})^2}\right] (v-v_{ab}^{(t)})^2.
		\end{aligned}
		\end{equation}
		Thus, to show $G(v,v_{ab}^{(t)}) \ge F_{ab}(v)$, we only need to show $\cref{eq_proof_v_G} \ge \cref{eq_proof_v_talor}$, which is equivalent to show
		\begin{equation}
		\label{eq_proof_v_aux}
		\begin{aligned}
		&\frac{(VU^TU)_{ab} + \lambda (DV)_{ab} + \frac{\beta}{2(1+v_{ab}^{(t)})} }{v_{ab}^{(t)}}  
		\ge &(U^TU)_{bb} + \lambda L_{aa} - \frac{\beta}{2(1+v_{ab}^{(t)})^2}.
		\end{aligned}
		\end{equation}
		With straightforward algebra, it is easy to see that
		\begin{equation}
		(VU^TU)_{ab} = \sum_{l=1}^{k} v_{al}^{(t)} (U^TU)_{lb} \ge v_{ab}^{(t)}(U^TU)_{bb},
		\end{equation}
		and
		\begin{equation}
		\begin{aligned}
		\lambda (DV)_{ab} =&  \lambda \sum_{l=1}^{n} D_{al} v_{lb}^{(t)} \ge \lambda D_{aa}v_{ab}^{(t)} \\
		\ge & \lambda (D-W)_{aa}v_{ab}^{(t)} = \lambda L_{aa} v_{ab}^{(t)}.
		\end{aligned}
		\end{equation}
		Thus,
		\begin{equation}
		\begin{aligned}
		&\frac{(VU^TU)_{ab} + \lambda (DV)_{ab} + \frac{\beta}{2(1+v_{ab}^{(t)})} }{v_{ab}^{(t)}} \\
		\ge &\frac{(VU^TU)_{ab} + \lambda (DV)_{ab}  }{v_{ab}^{(t)}} \\
		\ge & (U^TU)_{bb} + \lambda L_{aa} \\
		\ge & (U^TU)_{bb} + \lambda L_{aa} - \frac{\beta}{2(1+v_{ab}^{(t)})^2}.
		\end{aligned}
		\end{equation}
		Thus, $G(v,v_{ab}^{(t)})$ is an auxiliary function of $F_{ab}(v)$.
	\end{proof}
	Next, with the above \cref{def_aux,lemma_aux,lemma_proof_v_G}, we will prove \cref{thm_conv} in the following.
	\begin{proof}[Proof of \cref{thm_conv}]
		

		To obtain $v_{ab}^{(t+1)}$, we need to solve the following problem
		\begin{equation}
		\label{eq_proof_aux_solve}
		v_{ab}^{(t+1)} = \argmin_{v} G(v,v_{ab}^{(t)}).
		\end{equation}
		It is seen that $G(v,v_{ab}^{(t)})$ defined in \cref{eq_proof_v_G} is quadratic and convex.
		Thus, \cref{eq_proof_aux_solve} admits solution with first-order optimal condition:
		\begin{equation}
		\label{eq_proof_aux_cond}
		\begin{aligned}
		\!& F'_{ab}(v_{ab}^{(t)})  + 2\frac{(VU^TU)_{ab} \!+\! \lambda (DV)_{ab} \!+\! \frac{\beta}{2(1+v_{ab}^{(t)})} }{v_{ab}^{(t)}} (v-v_{ab}^{(t)})=0.
		\end{aligned}
		\end{equation}
		It is seen that\cref{eq_proof_aux_cond} leads to
		\begin{eqnarray}
		\label{eq_proof_aux_cond2}
		\begin{aligned}
		&  2\frac{(VU^TU)_{ab} + \lambda (DV)_{ab} + \frac{\beta}{2(1+v_{ab}^{(t)})} }{v_{ab}^{(t)}} \cdot v  \\
		= & 2{(VU^TU)_{ab} + 2\lambda (DV)_{ab} + \frac{\beta}{(1+v_{ab}^{(t)})} } - F'_{ab}(v_{ab}^{(t)})
		\end{aligned}
		\end{eqnarray}
		Hence,
		\begin{equation}
		\label{eq_proof_aux_cond3}
		\begin{aligned}
		v_{ab}^{(t+1)}	= & v_{ab}^{(t)} \!-\! F'_{ab}(v_{ab}^{(t)}) \frac{v_{ab}^{(t)}}{ 2(VU^TU)_{ab} \!+\! 2\lambda (DV)_{ab} \!+\! \frac{\beta}{(1+v_{ab}^{(t)})} } \\
		= & v_{ab}^{(t)}\frac{2(VU^TU)_{ab} \!+\! 2\lambda (DV)_{ab} \!+\! \frac{\beta}{(1+v_{ab}^{(t)})} 
			\!-\! F'_{ab}(v_{ab}^{(t)}) }{ 2(VU^TU)_{ab} \!+\! 2\lambda (DV)_{ab} \!+\! \frac{\beta}{(1+v_{ab}^{(t)})} } \\
		= & v_{ab}^{(t)}\frac{ 2(X^TU)_{ab} + 2\lambda (WV)_{ab} }{ 2(VU^TU)_{ab} + 2\lambda (DV)_{ab} + \frac{\beta}{(1+v_{ab}^{(t)})} },
		\end{aligned}
		\end{equation}
		which essentially results in the updating rule of \cref{eq_update_v}.
		Since $G(v,v_{ab}^{(t)})$ is an auxiliary function for $F_{ab}(v)$, \cref{eq_proof_aux_cond3} guarantees the non-increasing property of $F_{ab}(v)$.
		Hence, the objective $\mathcal{O}$ is non-increasing under the update rule of \cref{eq_update_v}.
		
		Mathematically, the matrices $U$ and $V$ are playing similar rules in the model and thus the proof regarding \cref{eq_update_u} follows \cref{eq_update_v}.
		We only need to replace $X$ with $X^T$ and set $\lambda=0$, then the above analysis applies to \cref{eq_update_u}, which concludes the proof.

	\end{proof}
	
	{
		\begin{rem}
			In the above analysis, it is been shown that the value of objective function is decreasing with the alternative updating rules of $U$ and $V$.
			We define $\varUpsilon = [U^T,V^T]^T\in\mathcal{R}^{(d+n) \times k}$ and treat the updates of \cref{eq_update_u} and \cref{eq_update_v} as a mapping $\varUpsilon^{(t+1)} = \mathcal{M}(\varUpsilon^{(t)})$.
			Then, clearly we have $\varUpsilon^* = \mathcal{M}(\varUpsilon^*)$ at convergence.
			Following \cite{ding2010convex,xu1996convergence}, with non-negativity constraint enforced, we expand $\varUpsilon \approxeq \mathcal{M}(\varUpsilon^*) + (\partial \mathcal{M} / \partial \varUpsilon)(\varUpsilon - \varUpsilon^*)$, which indicates that $\|\varUpsilon^{(t+1)} -\varUpsilon^*\|\le \|\partial \mathcal{M} / \partial \varUpsilon\| \cdot \|\varUpsilon^{(t)} -\varUpsilon^*\|$ under an appropriate matrix norm. 
			In fact, $\|\partial \mathcal{M} / \partial \varUpsilon\|\not= 0$ generally holds. Thus, \cref{eq_update_u} and \cref{eq_update_v} roughly have a first-order convergence rate.
			
		\end{rem}
	}

	\section{Robust Log-norm Regularized Sparse NMF}
	\label{sec_robust}
	In this section, we further develop a robust model based on the LS-NMF, which is named robust log-norm regularized sparse NMF model (RLS-NMF).
	In particular, with the fundamentals of the LS-NMF in \cref{sec_proposed}, we will present the detailed formulation, optimization, and convergence analysis for RLS-NMF in this section.
	\subsection{Formulation of RLS-NMF}
	It is noted that the LS-NMF model raised in \cref{sec_proposed} seeks the nonnegative representation with original data.
	Unfortunately, data are often observed and collected with noise, which severely degrades the learning performance. 
	Thus, there is a demanding need to develop more robust model to handle noise effects and promote the learning performance. 
	To enhance the robustness, we adopt the more robust measure $\ell_{2,1}$-norm to minimize the residual instead of the Frobenius-norm, which leads to
	\begin{equation}
	\label{eq_obj_log_robust1}
	\begin{aligned}
	& \min_{U,V}  \|X-UV^T\|_{2,1} + \lambda \textbf{Tr}(V^TLV)  + \alpha \|U\|_{\log} + \beta \|V\|_{\log} \\
	& \quad s.t. \quad u_{ij}\ge 0, v_{ij} \ge 0,
	\end{aligned}
	\end{equation}
	where, for a matrix $M$, $\|M\|_{2,1}=\sum_{j} \|m_j\|_2$ is the $\ell_{2,1}$-norm with column-wise sparsity.
	Here, the $\ell_{2,1}$-norm is invariant and helps keep spatial information of the examples.
	However, the optimization of $\ell_{2,1}$ with nonnegative constraints is difficult.
	To facilitate the optimization, we further decompose the data as $X=UV^T+S$, where the matrix $S$ with column-wise sparsity is introduced to account for the noise.
	With the above assumption, we relax model \cref{eq_obj_log_robust1} to the following
	\begin{equation}
	\label{eq_obj_log_robust2}
	\begin{aligned}
	\min_{U,V,S} &  \|X-S-UV^T\|_F^2 + \gamma \|S\|_{2,1}  \\
	&+ \lambda \textbf{Tr}(V^TLV) + \alpha \|U\|_{\log} + \beta \|V\|_{\log}\\
	s.t. & \quad u_{ij}\ge 0, v_{ij} \ge 0,
	\end{aligned}
	\end{equation}
	where $\gamma\ge0$ is a balancing parameter.
	It is seen that the relaxed model \cref{eq_obj_log_robust2} is easier to solve and the balancing parameter $\gamma$ allows the model to have more freedom.
	It is noted that the $\ell_{2,1}$-norm is defined as the summation of $\ell_2$-norms of all column vectors in a matrix,
	where the summation actually performs in a way similar to $\ell_1$-norm.
	As pointed out in earlier section that the $\ell_1$-norm might be less efficient in approximating the true sparsity,
	we design the following novel $\ell_{2,\log}$-(pseudo) norm to better restrict column-wise sparsity:
	\begin{equation}
	\label{eq_l2log}
	\|M\|_{2,\log} = \sum_{j} \log(1 + \|m_j\|_2).
	\end{equation}
	{
		For any type of noise, the expectation of the sparsity measurement by \cref{eq_l2log} is less than the $\ell_{2,1}$-based measurement.
		Our explanation of the above statement is as follows. 
		Let $c$ be a column of $S$ and we denote the elements of $c$ by $c_1, c_2, \cdots, c_d.$
		For any types of distribution of $c_i$ for $i=1,\cdots,d$, 
		the expectation of the log-based approximation is generally less than the $\ell_2$-based approximation.
		Let $\mathbf{E}(\cdot)$ be the expectation and $f_{\sum_{i=1}^d c_{i}^2}(y)$ be the probability density function for $y=\sum_{i=1}^d c_i^2$,
		then the above conclusion can be formally analyzed in the following way:
		\begin{equation}
		\begin{aligned}
		& \mathbf{E} \Bigg(\log \Bigg(1 + \sqrt{\sum\nolimits_{i=1}^d c_{i}^2}\Bigg) \Bigg) =  \int_{0}^{+\infty} \log  (1 + \sqrt{ y } ) f_{\sum_{i=1}^d c_{i}^2}(y) \dd{y} \\
		& < \int_{0}^{+\infty} \sqrt{y} f_{\sum_{i=1}^d c_{i}^2}(y)  \dd{y} = \mathbf{E} \Bigg(\sqrt{\sum\nolimits_{i=1}^d c_{i}^2} \quad \Bigg).
		\end{aligned}
		\end{equation}
		The above inequality generally holds for all columns of $S$, i.e., $s_1,\cdots,s_n$, thus it is straightforward that
		\begin{equation}
		\begin{aligned}
		\mathbf{E} \Big( \|S\|_{2,\log} \Big) & = \mathbf{E} \Big( \sum\nolimits_{i=1}^{n} \log(1 + \|s_i\|_2) \Big) = \sum\nolimits_{i=1}^{n} \mathbf{E}(\log(1 + \|s_i\|_2))  \\
		& < \sum\nolimits_{i=1}^{n} \mathbf{E}( \|s_i\|_2) = \mathbf{E} \Big(\sum\nolimits_{i=1}^{n} \|s_i\|_2 \Big) = \mathbf{E}(\|S\|_{2,1}). 
		\end{aligned}
		\end{equation}
		%
		Moreover, if $s_i$ only contains essentially small values, then it is natural that $s_i$ contains noise and is indeed sparse.
		Thus, for such a column of $S$, it is essentially important that the approximation is close to 0 rather than 1 to distinguish noise effects and useful information.
		It is noted that $\log(1+\sqrt{x})<\sqrt{x}$ holds for small $x>0$,
		which indicates that the log-based approximation is closer to 0 than the $\ell_2$-based approach and thus is more accurate in approximating the real sparsity. 
		Thus, it is expected that the log-based approximation is more accurate in approximating the real sparse indicator of the columns than the $\ell_2$-based approach.	
		
	}
	
	We incorporate the $\ell_{2,\log}$-(pseudo) norm into \cref{eq_obj_log_robust2}, which leads to the robust log-norm regularized sparse NMF model (RLS-NMF):
	\begin{equation}
	\label{eq_obj_log_robust}
	\begin{aligned}
	\min_{U,V,S} &  \|X-S-UV^T\|_F^2 + \gamma \|S\|_{2,\log}  \\
	&+ \lambda \textbf{Tr}(V^TLV) + \alpha \|U\|_{\log} + \beta \|V\|_{\log}\\
	s.t. & \quad u_{ij}\ge 0, v_{ij} \ge 0.
	\end{aligned}
	\end{equation}
	For the optimization and convergence analysis of the RLS-NMF model in \cref{eq_obj_log_robust}, we will present them in details in rest of this section.
	
	%
	%
	%
	
	\subsection{Optimization of RLS-NMF}
	For $S$-minimization, the sub-problem is
	\begin{equation}
	\label{eq_sub_s}
	\begin{aligned}
	\min_{S} &  \|X - S - UV^T\|_F^2 + \gamma \|S\|_{2,\log}.
	\end{aligned}
	\end{equation}
	We formally provide the following theorem to solve this type of optimization problem.
	
	\begin{theorem}[$\ell_{2,\log}$-shrinkage operator]
		Given matrix $Y$ and a nonnegative parameter $\tau$, the following problem
		\begin{equation}
		\label{eq_soft_thres}
		\min_{W} \frac{1}{2} \|Y-W\|_F^2 + \tau \|W\|_{2,\log}
		\end{equation}
		%
		admits closed-form solution in a column-wise manner:
		\begin{equation}
		\label{eq_sol_soft_thres}
		\!{ \!\!w_i =
			\begin{cases}
			\frac{\xi}{\|y_i\|_2}y_i,& \!\mbox{ if $f_i(\xi) \!\le\! \frac{1}{2}\|y_i\|_2^2$, $ (1 + \|y_i\|_2)^2 > 4\tau$, $\xi >0$}\\
			0, &\mbox{ otherwise, }
			\end{cases} }
		\end{equation}
		where $f_i(x) = \frac{1}{2}(x-\|y_i\|_2)^2 + \tau \log (1+x),$ and $\xi = \frac{\|y_i\|_2-1}{2} + \sqrt{\frac{(1+\|y_i\|_2)^2}{4} - \tau }.$
		
	\end{theorem}
	\begin{proof}
		The objective of \cref{eq_soft_thres} can be rewritten in a column-wise manner as
		\begin{equation}
		\label{eq_soft_thres_col}
		\min_{w_i} \sum_{i=1}^{n} \Big\{\frac{1}{2} \|y_i-w_i\|_2^2 + \tau \operatorname{log}( 1+\|w_i\|_2 ) \Big\},
		\end{equation}
		such that each $w_i$ can be obtained by column independently.
		For $w_i$, the subproblem is
		\begin{equation}
		\label{eq_soft_thres_wi}
		\min_{w_i} \frac{1}{2} \|y_i-w_i\|_2^2 + \tau \operatorname{log}( 1+\|w_i\|_2 ).
		\end{equation}
		We may treat $w_i$ as a special matrix and perform thin SVD to it.
		Then it is seen that $w_i$ has exactly one singular value, which is $\operatorname{\sigma}(w_i) = \sqrt{w_i^Tw_i} = \|w_i\|_2$,
		where $\operatorname{\sigma}(\cdot)$ is the singular value of the input vector.
		Thus, \cref{eq_soft_thres_wi} is equivalent to
		\begin{equation}
		\label{eq_soft_thres_wi_svd}
		\min_{w_i} \frac{1}{2} \|y_i-w_i\|_2^2 + \tau \operatorname{log}( 1+\sigma(w_i) ).
		\end{equation}
		Hence, according to \cite{peng2015subspace}, the solution to \cref{eq_soft_thres_wi_svd} is obtained with
		\begin{equation}
		w_i = u_i \sigma^*(w_i) v_i^T,
		\end{equation}
		where $u_i$ and $v_i$ are left and right singular vectors of $y_i$ and $\sigma^*(w_i)=\argmin_{x\ge 0} \frac{1}{2} (\sigma(y_i)-x)^2 + \tau \operatorname{log}( 1+x ).$
		Thus, by solving the equation, we have
		\begin{equation}
		\!\!\sigma^*(w_i)\!\! =\!\!
		\begin{cases}
		\!\xi ,& \! \mbox{ if $f_i(\xi) \!\le\! f_i(0)$, $ (1 + \sigma(y_i))^2 \!>\! 4\tau$, $\xi >0$, }\\
		\! 0, & \! \mbox{ otherwise, }
		\end{cases}
		\end{equation}
		with $f_i(x) = \frac{1}{2}(x-\sigma(y_i))^2 + \tau \log (1+x),$ and $\xi = \frac{\sigma(y_i)-1}{2} + \sqrt{\frac{(1+\sigma(y_i))^2}{4} - \tau }.$
		It is straightforward that $y_i = \frac{y_i}{\|y_i\|_2}\|y_i\|_2 [1]$ is a thin SVD of $y_i$, where [1] is a special matrix with 1 being the only one element.
		We substitute $u_i \!=\! \frac{y_i}{\|y_i\|_2}$, $\sigma(y_i) \!=\! \|y_i\|_2$, and $v_i = [1]$ into above equations,
		which leads to \cref{eq_sol_soft_thres} and concludes the proof.
	\end{proof}
	
	{
		
		\begin{rem}
			Regarding the problem \cref{eq_soft_thres}, it is easy to verify that for a given $Y$, a larger $\tau$ generally leads to a potentially sparser solution for $W$. 
			To see this, we consider the three conditions given in \cref{eq_sol_soft_thres}. 
			We only consider the case $0\le \tau < \frac{(1+\|y_i\|_2)^2}{4}$, since $\tau \ge \frac{(1+\|y_i\|_2)^2}{4}$ directly returns a zero matrix as the solution for $W$.
			Then, it is straightforward to see that for $\tau'>\tau$, the corresponding $\xi'= \frac{\|y_i\|_2-1}{2} + \sqrt{\frac{(1+\|y_i\|_2)^2}{4} - \tau'}<\xi.$
			Thus, given $Y$, $\textit{Prob}(\xi'>0|Y) < \textit{Prob}(\xi>0|Y)$, implying that the third condition is more difficult to satisfy for $\tau'$. 
			For the first condition, our analysis is as follows. Let $e=\frac{1+\|y_i\|_2}{2}$, then
			\begin{equation}
			\label{eq_sparse_analysis_f}
			\begin{aligned}
			f_i(\xi) &= \frac{1}{2}(\xi-\|y_i\|_2)^2 + \tau \log (1+\xi)	\\
			& = \frac{1}{2}(- e + \sqrt{e^2-\tau})^2 + \tau \log (e + \sqrt{e^2-\tau})	\\
			& = (e^2 - e\sqrt{e^2-\tau} - \frac{\tau}{2}) + \tau \log (e+\sqrt{e^2-\tau}).
			\end{aligned}
			\end{equation}
			We treat \cref{eq_sparse_analysis_f} as a function of $\tau$ and let $g(\tau)=f_i(\xi)$, then it is seen that 
			\begin{equation}
			\label{eq_sparse_analysis_g}
			\begin{aligned}
			g'(\tau) &= \frac{e}{2\sqrt{e^2-\tau}} - \frac{1}{2} + \log(e+\sqrt{e^2-\tau}) +  \frac{\tau}{e+\sqrt{e^2-\tau}}\cdot\frac{-1}{2\sqrt{e^2-\tau}} \\
			& = \frac{e(e+\sqrt{e^2-\tau}) - \sqrt{e^2-\tau}(e+\sqrt{e^2-\tau}) -\tau}{2\sqrt{e^2-\tau}(e+\sqrt{e^2-\tau})}+ \log(e+\sqrt{e^2-\tau}) \\
			& = \frac{e^2+e\sqrt{e^2-\tau} - e\sqrt{e^2-\tau} -e^2 +\tau -\tau}{2\sqrt{e^2-\tau}(e+\sqrt{e^2-\tau})}+ \log(e+\sqrt{e^2-\tau}) \\
			& = \log(e+\sqrt{e^2-\tau}) \\
			& = \log\big((1+\|y_i\|_2)/2+\sqrt{(1+\|y_i\|_2)^2/4-\tau}\big) = \log(1+\xi) >0.
			\end{aligned}
			\end{equation}
			Thus, it is straightforward that $g(\tau') > g(\tau)$ for $\tau'>\tau$,
			which implies that $\textit{Prob}(g(\tau')\le \frac{\|y_i\|_2^2}{2}|Y) < \textit{Prob}(g(\tau)\le \frac{\|y_i\|_2^2}{2}|Y)$.
			Thus, the first condition is also more difficult to satisfy for $\tau'$. 
			In summary, it is seen that the conditions in \cref{eq_sol_soft_thres} are more difficult to satisfy for a larger value of $\tau$,
			which suggests that a larger $\tau$ potentially leads to a larger number of zero columns for $w$ and thus leads to a potentially sparser solution.
			
		\end{rem}
		
	}
	
	For ease of representation, we denote the $\ell_{2,\log}$-shrinkage operator in \cref{eq_sol_soft_thres} as $\mathcal{S}_{\tau}(Y)$. Then $S$ admits a closed-form solution with the $\ell_{2,\log}$-shrinkage operator:
	\begin{equation}
	\label{eq_update_s}
	S = \mathcal{S}_{\frac{\gamma}{2}}(X-UV^T).
	\end{equation}
	To solve $U$ and $V$, the associated sub-problem is
	\begin{equation}
	\label{eq_obj_log_robust_uv}
	\begin{aligned}
	\min_{U,V} &  \|X-S-UV^T\|_F^2 + \lambda \textbf{Tr}(V^TLV) + \alpha \|U\|_{\log}  \\
	& + \beta \|V\|_{\log} \quad s.t. \quad u_{ij}\ge 0, v_{ij} \ge 0.
	\end{aligned}
	\end{equation}
	It is seen that the above problem is similar to \cref{eq_obj_log} except that the factorization is performed on $X-S$ in \cref{eq_obj_log_robust_uv} instead of $X$.
	To derive updating rules for $U$ and $V$ in a similar way to \cref{eq_update_u,eq_update_v}, we first provide the following theorem to guarantee the non-negativity of $X-S$, 
	which is essential for the optimization and nonnegativity of $U$ and $V$.
	
	\begin{theorem}
		\label{thm_nonnegative_xs}
		Given nonnegative data $X$ and values of $U$ and $V$, the matrix $X-S$ is nonnegative under the updating rule of \cref{eq_update_s}.
	\end{theorem}
	\begin{proof}
		We denote $M=X-UV^T$, then it is easy to see that the optimal $s_{j}$ is either a zero vector or scaled $m_j$ with a positive scaling factor $\xi/\|m_i\|_2$.
		We consider the following two cases.
		
		1) For the columns that $s_j=0$, it is easy to verify that $x_j - s_j = x_j$, which is nonnegative.
		
		2) For the columns that $s_j\not=0$, it is easy to see that 
		\begin{equation}
		\begin{aligned}
		\xi & = \frac{\|m_j\|_2-1}{2} + \sqrt{\frac{(1+\|m_j\|_2)^2}{4} - \tau } \\ 
		& \le \frac{\|m_j\|_2-1}{2} + \sqrt{\frac{(1+\|m_j\|_2)^2}{4}}  = \frac{\|m_j\|_2-1}{2} + \frac{1+\|m_j\|_2}{2} = \|m_j\|_2.
		\end{aligned}
		\end{equation}
		Thus, it is seen that the corresponding $s_j$ is obtained by scaling $m_j$ with a factor $\frac{\xi}{\|m_j\|_2}\le 1$, which indicates that $x_j - s_j = m_j - s_j + (UV^T)_j = \frac{\|m_j\|_2 - \xi}{\|m_j\|_2} + (UV^T)_j$ is nonnegative.
		
		It is seen that all columns of $X-S$ are nonnegative and thus the matrix $X-S$ is nonnegative.
	\end{proof}
	\cref{thm_nonnegative_xs} is important in that it guarantees the nonnegativity of $U$ and $V$ with the following updating rules,
	which is essential to the nature of parts-based representation:
	\begin{align}
	\label{eq_update_u_robust}
	\!u_{ij} \!&\! \leftarrow\! u_{ij} \frac{(2(X-S)V)_{ij}}{(2UV^TV  + \alpha\textbf{1}_{d\times k}\oslash(\textbf{1}_{d\times k}+U) )_{ij} } \\
	\label{eq_update_v_robust}
	\!v_{ij} \!&\! \leftarrow \! v_{ij}\! \frac{ 2((X\!-\!S)^TU \!+\! \lambda WV)_{ij}}{(2VU^TU \!+\! 2\lambda DV \!+\! \beta \textbf{1}_{n\times k} \!\oslash\! (\textbf{1}_{n\times k} + V))_{ij}}.
	\end{align}

	\subsection{Convergence Analysis for RLS-NMF}
	\label{sec_conv_rsl}
	For the updating rules of \cref{eq_update_s,eq_update_u_robust,eq_update_v_robust}, it is guaranteed that the objective function value sequence converges.
	We formally provide the theoretical result with the following theorem.
	\begin{theorem}
		Given nonnegative initial values of $U$ and $V$, the objective function of \cref{eq_obj_log_robust} is monotonally decreasing under the updating rules of \cref{eq_update_s,eq_update_u_robust,eq_update_v_robust}.
	\end{theorem}
	\begin{proof}
		Given nonnegative initial values of $U$ and $V$, $X-S$ is nonnegative.
		Then the proof of $U$ and $V$ follows \cref{thm_conv} by replacing $X$ with $X-S$.
		Thus, the objective function is non-increasing under the updating rules of \cref{eq_update_u_robust,eq_update_v_robust}.
		For the updating rule of \cref{eq_update_s}, since it is the optimal solution to \cref{eq_sub_s}, the objective function is guaranteed to be non-increasing.
		Due to the nonnegativity of the objective function \cref{eq_obj_log_robust}, the value sequence must converge under the updating rules \cref{eq_update_s,eq_update_u_robust,eq_update_v_robust}.
	\end{proof}

	\section{Experiments}
	\label{sec_exp}
	In this section, we conduct extensive experiments to testify the effectiveness of the proposed method.
	In particular, we compare our method with several state-of-the-art NMF methods,
	including NMF \cite{lee1999learning}, weighted NMF (WNMF) \cite{kim2009weighted}, orthogonal NMF (ONMF) \cite{ding2006orthogonal}, convex NMF (CNMF) \cite{ding2010convex}, graph regularized NMF (GNMF) \cite{cai2011graph}, robust manifold NMF (RMNMF) \cite{huang2014robust}, and semi NMF (Semi-NMF) \cite{ding2010convex}.
	We testify all methods on 8 widely used benchmark data sets, including Yale \cite{belhumeur1997eigenfaces}, Jaffe \cite{lyons1998japanese}, ORL \cite{samaria1994parameterisation}, AR \cite{martinez1998ar}, Extended Yale B (EYaleB) \cite{georghiades2001few}, COIL20 \cite{Nene1996Columbia}, Pendigits \cite{Bache+Lichman:2013}, and Semeion \cite{semeion2014}.
	All examples of these data sets are scaled to have a unit $\ell_2$ norm. 
	Three evaluation metrics are used in the experiment, including clustering accuracy, normalized mutual information (NMI), and purity.
	{
		All these metrics have values ranging within $[0,1]$, where the higher values represent better clustering results.
		For these metrics, we briefly introduce them in the following.
		Clustering accuracy measures the extent to which each cluster contains data points from the same class.
		It is defined as 
		\begin{equation}
		\label{eq_acc}
		\text{Accuracy} = \frac{\sum_{i=1}^{n}\delta(map(r_i),l_i)}{n},
		\end{equation} 
		where $n$ is the total number of data points, 
		$r_i$ and $l_i$ are the predicted and true labels of the data point $x_i$, respectively, 
		$\delta(a,b)$ is a delta function that returns 1 when $a = b$ otherwise 0,
		and $map(r_i)$ is a mapping function that maps $r_i$ to an equivalent label by permutation such that \cref{eq_acc} is maximized. 
		Normalized mutual information measures the quality of the clusters, which is defined as 
		\begin{equation}
		\label{eq_nmi}
		\text{Normalized mutual information} = \frac{\sum_{i=1}^{N}\sum_{i=1}^{N}n_{i,j}\log\frac{n_{i,j}}{n_i \hat{n}_j}}{\sqrt{(\sum_{i=1}^{N}n_i\log\frac{n_i}{n})(\sum_{j=1}^{N}\hat{n}_j\log\frac{\hat{n}_j}{n}})},
		\end{equation} 
		where $N$ is the number of clusters, $n_i$ and $\hat{n}_j$ denote the sizes of the $i$-th cluster and $j$-th class, respectively, 
		and $n_{i,j}$ denotes the number of data points in the intersection between them. 
		Purity is a simple and transparent evaluation measure, which measures the extent to which each cluster contains data points from primarily one class. 
		It is defined as 
		\begin{equation}
		\label{eq_purity}
		\text{Purity} = \frac{1}{n}\sum_{i=1}^{N}\max_j (n_{i,j}).
		\end{equation} 
		
	}

	\begin{figure}[!tb]
		\centering{
			\includegraphics[width=0.9\columnwidth]{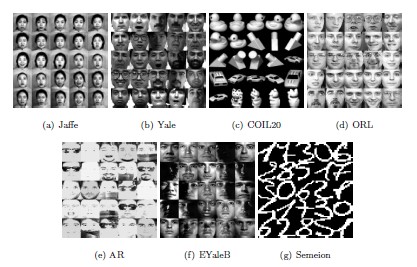}
			\caption{ Examples of the data sets used in our experiments. Because Pendigits data set has low resolution and it is hard to observe visual feature details, we do not show examples from this data set. }
		\label{fig_imgs}
		}
	\end{figure}

	For all methods in comparison, we follow the following settings for parameters.
	We tune all the balancing parameters within the set $\{0.001,0.01,0.1,1,10,$ $100,1000\}$ for all methods.
	For the graph Laplacian that is used in methods such as GNMF, RMNMF, and RLS-NMF,
	without loss of generality, we use the binary weighting strategy with 5 neighbors kept in the similarity graph matrix.
	For all methods in comparison, the exact number of clusters of the data is provided to determine $k$, which follows a common setting in literature \cite{peng2020nonnegative,cai2011graph}.
	After we obtain the factorization from each method, K-means clustering is performed on the representation matrix $V$ to obtain the final clustering result.
	For all methods, we tune the parameters with all possible combinations and report the best performance.

	\begin{table*}[h] 
		\scriptsize
		\centering
		\caption{Clustering Performance on 8 Benchmark Data Sets}
		\resizebox{0.95\textwidth}{0.3\textwidth}{
			\begin{tabular}{|p{1.6cm}<{\centering} | p{1.1cm}<{\centering} | p{1.1cm}<{\centering} | p{1.2cm}<{\centering} | p{1.1cm}<{\centering} | p{1.1cm}<{\centering} | p{1.1cm}<{\centering}  | p{1.4cm}<{\centering} | p{1.3cm}<{\centering} | p{1.3cm}<{\centering} |}
				\hline
				\multirow{2}{0.6cm}{\centering Data}& \multicolumn{9}{c|}{Accuracy (\%)}\\	
				\cline{2-10}
				\multirow{2}{0.6cm}{}
				& NMF  			& GNMF				& RMNMF		 		& WNMF				& CNMF 				& ONMF				&	Semi-NMF		 & LS-NMF           & RLS-NMF \\ \hline
				Semeion	& 52.86	& \textcolor{green}{63.03}	& 43.57	& 55.56	& 45.20	& 50.09	& 57.94	& \textcolor{blue}{68.49}  &\textcolor{red}{68.86}	\\
				EYaleB	& 12.34	& 16.16	& 16.86	& 14.79	& 9.61	& 11.43	& \textcolor{blue}{24.61}	& \textcolor{green}{18.64}    & \textcolor{red}{37.20}	\\
				ORL	    & 53.50	& 55.75	& 56.25	& 52.25	& 23.00	& 49.00	& \textcolor{green}{57.25}	& \textcolor{blue}{62.25} & \textcolor{red}{68.50}	\\
				AR 	    & 10.46	& 22.69	& 27.54	& 26.00	& 11.92	& 22.85	& \textcolor{red}{32.62} & \textcolor{green}{26.85} & \textcolor{blue}{29.23}	\\
				Jaffe	& 85.45	& \textcolor{blue}{97.65} &\textcolor{green}{95.77}	& 90.61	& 69.95	& 82.63	& \textcolor{blue}{97.65} & \textcolor{red}{98.59} &\textcolor{red}{98.59}	\\
				Yale	& 21.82	& \textcolor{green}{43.03} & \textcolor{blue}{44.85} & \textcolor{blue}{44.85}	& 40.00	& 41.21	& 40.00	& \textcolor{red}{48.48} &\textcolor{red}{48.48}	\\
				COIL20	& \textcolor{green}{67.43}	& \textcolor{blue}{82.71} & 56.39 & 63.40	& 56.87	& 57.36	& 61.74	&\textcolor{red}{85.97} &\textcolor{red}{85.97}\\
				Pendigits	& 77.97	& \textcolor{green}{79.20}	& 49.55	& 73.62	& 60.12	& 58.58 & 72.78 & \textcolor{blue}{88.16} & \textcolor{red}{88.26}	\\	\hline
				\hline
				\multirow{2}{0.6cm}{\centering Data}& \multicolumn{9}{c|}{Normalized Mutual Information (\%)} \\	
				\cline{2-10}
				\multirow{2}{0.6cm}{}
				& NMF  			& GNMF				& RMNMF		 		& WNMF				& CNMF 				& ONMF				&	Semi-NMF		 & LS-NMF           & RLS-NMF \\ \hline
				Semeion	& 43.50	&\textcolor{green}{58.88} & 35.44 & 44.82 & 37.96	& 40.12	& 47.91	& \textcolor{blue}{61.69} &  \textcolor{red}{62.98}	\\
				EYaleB	& 21.59	& 25.86	& 28.46	& 26.73	& 14.39	& 16.32	& \textcolor{blue}{43.32} & \textcolor{green}{29.12}  & \textcolor{red}{44.11}	\\
				ORL	    & 74.51	& 74.72	& 73.08	& 72.78	& 43.51	& 69.76	& \textcolor{green}{75.39}	&\textcolor{blue}{76.41} & \textcolor{red}{81.44}	\\
				AR  	& 25.71	& 43.49	& 44.04	& 43.43	& 27.12	& 41.23 & \textcolor{red}{49.46} & \textcolor{green}{44.53} & \textcolor{blue}{44.72} 	\\
				Jaffe	& 85.40	& \textcolor{blue}{96.50} & 93.54 & 89.44 & 70.65 & 84.46	& \textcolor{green}{96.48} &\textcolor{red}{98.13} & \textcolor{red}{98.13} \\
				Yale	& 29.25	& \textcolor{green}{48.34}	& 48.20	& \textcolor{blue}{51.09} & 41.37 & 43.99 & 47.52 & \textcolor{red}{51.50} & \textcolor{red}{51.50} \\
				COIL20	& \textcolor{green}{76.00}	& \textcolor{red}{90.59} & 66.65 & 72.73 & 70.53 & 70.94	& 73.92	&\textcolor{blue}{90.32} &\textcolor{blue}{90.32}\\
				Pendigits	& 71.09		& \textcolor{green}{73.02}		& 40.48		& 66.07	& 60.29	& 58.84 & 67.27	& \textcolor{blue}{83.88} & \textcolor{red}{84.06}	\\	\hline
				\hline	
				\multirow{2}{0.6cm}{\centering Data}& \multicolumn{9}{c|}{Purity (\%)}\\	
				\cline{2-10}
				\multirow{2}{0.6cm}{}
				& NMF  			& GNMF				& RMNMF		 		& WNMF				& CNMF 				& ONMF				&	Semi-NMF		 & LS-NMF           & RLS-NMF \\ \hline
				Semeion	& 54.05	&\textcolor{green}{65.29} & 45.95	& 56.56	& 45.20	& 54.74	& 58.82	&\textcolor{blue}{68.49}    &\textcolor{red}{68.86}	\\
				EYaleB	& 13.01	& 18.39	& 17.94	& 15.29	& 10.65	& 09.32	& \textcolor{blue}{25.43}	&\textcolor{green}{19.26}  & \textcolor{red}{38.19}	\\
				ORL	    & 60.50	& 62.25	& 61.00	& 58.50	& 25.00	& 47.25	& \textcolor{green}{63.25}	& \textcolor{blue}{65.75}    &\textcolor{red}{72.25}	\\
				AR	    & 11.08	& 25.23	&\textcolor{green}{29.31} & 28.92 & 13.00	& 20.38	& \textcolor{red}{35.38} & 29.00 &\textcolor{blue}{31.62} 	\\
				Jaffe	& 85.45	& \textcolor{blue}{97.65} &\textcolor{green}{95.77}	& 90.61	& 74.18	& 82.36	& \textcolor{blue}{97.65}	&\textcolor{red}{98.59} &\textcolor{red}{98.59}\\
				Yale	& 26.06	& 44.24	& \textcolor{green}{44.85}	&\textcolor{blue}{47.27} & 40.61	& 40.61	& 41.21	&\textcolor{red}{48.48}    &\textcolor{red}{48.48}\\
				COIL20	& \textcolor{green}{69.24}	&\textcolor{blue}{84.44} & 58.13	& 64.65	& 60.07	& 60.14	& 63.61	&\textcolor{red}{86.25} &\textcolor{red}{86.25}\\
				Pendigits	& 77.97		& \textcolor{green}{79.20}	& 49.57	& 73.62	& 65.78	& 65.01 & 72.78	&\textcolor{blue}{88.16} & \textcolor{red}{88.26}	\\	\hline
				
			\end{tabular}
		}
		\\The top three performances are highlighted in {\color{red}red}, {\color{blue}blue}, and {\color{green}green}, respectively.
		\label{tab_all}
	\end{table*}

	\begin{table*}[t] 
		\scriptsize
		\centering
		\caption{{Clustering Performance on Corrupted Jaffe Data Set}}
		\resizebox{0.95\textwidth}{0.18\textwidth}{
			\begin{tabular}{|p{1.6cm}<{\centering} | p{1.1cm}<{\centering} | p{1.1cm}<{\centering} | p{1.2cm}<{\centering} | p{1.1cm}<{\centering} | p{1.1cm}<{\centering} | p{1.1cm}<{\centering}  | p{1.4cm}<{\centering} | p{1.3cm}<{\centering} | p{1.3cm}<{\centering} |}
				\hline
				\multirow{2}{1.2cm}{\centering Corruption Level }& \multicolumn{9}{c|}{Accuracy (\%)}\\	
				\cline{2-10}
				\multirow{2}{0.6cm}{}
				& NMF  			& GNMF				& RMNMF		 		& WNMF				& CNMF 				& ONMF				&	Semi-NMF		 & LS-NMF           & RLS-NMF \\ \hline		
				20\%    &77.00 &\textcolor{green}{89.20} &76.53 & 84.51 & 75.59	& 68.08	& \textcolor{blue}{90.61} &\textcolor{green}{89.20}  & \textcolor{red}{92.96}\\
				40\%	&61.50 &\textcolor{blue}{77.93} &65.73 & 67.14 & 56.34	&58.69	& 63.38	&\textcolor{green}{76.53}	 &\textcolor{red}{80.75}	\\
				60\%	&52.11 &\textcolor{green}{57.75} &55.40 & 50.23	& 48.83	& 43.66	& 50.23	&\textcolor{blue}{68.54}  &\textcolor{red}{69.95}  \\
				\hline
				\hline
				\multirow{2}{1.2cm}{\centering Corruption Level }& \multicolumn{9}{c|}{Normalized Mutual Information (\%)}\\	
				\cline{2-10}
				\multirow{2}{0.6cm}{}
				& NMF  			& GNMF				& RMNMF		 		& WNMF				& CNMF 				& ONMF				&	Semi-NMF		 & LS-NMF           & RLS-NMF \\ \hline			
				20\%  	&79.30 &\textcolor{blue}{87.52} &75.51 & 81.80	& 76.80	& 70.93	& 86.28	&\textcolor{green}{86.82} &\textcolor{red}{89.31}    \\
				40\%	& 58.66&\textcolor{green}{73.20} & 62.11 & 69.84	& 57.71	& 54.51	& 60.30	& \textcolor{blue}{73.94} & \textcolor{red}{78.29}	\\
				60\%	& 49.25 &\textcolor{green}{59.07} & 50.11 & 47.97	& 44.75	& 41.26	& 47.96	& \textcolor{blue}{62.63} & \textcolor{red}{66.41} \\
				\hline\hline
				\multirow{2}{1.2cm}{\centering Corruption Level }& \multicolumn{9}{c|}{Purity (\%)}\\	
				\cline{2-10}
				\multirow{2}{0.6cm}{}
				& NMF  			& GNMF				& RMNMF		 		& WNMF				& CNMF 				& ONMF				&	Semi-NMF		 & LS-NMF           & RLS-NMF \\ \hline

				20\%   & 80.75 & \textcolor{green}{89.20} & 79.34 & 84.51	& 76.53	& 72.77	& \textcolor{blue}{90.61} & \textcolor{green}{89.20}   &\textcolor{red}{92.96}   	\\
				
				40\%	& 62.44 & \textcolor{blue}{77.93} & 67.61 & 71.83	& 61.50	& 59.62	& 63.38	&\textcolor{green}{76.53}  &\textcolor{red}{80.75}	\\
				60\%	& 54.93 & \textcolor{green}{60.56} & 56.34 & 54.46	& 50.23	& 46.01	& 52.11	& \textcolor{blue}{69.48}  &\textcolor{red}{69.95}	\\ \hline
				
			\end{tabular}
		}
		\label{tab_per_noisy_jaffe}
	\end{table*}
	
	\begin{table*}[t] 
		\scriptsize
		\centering
		\caption{Clustering Performance on Corrupted Yale Data Set}
		\resizebox{0.95\textwidth}{0.18\textwidth}{
			\begin{tabular}{|p{1.6cm}<{\centering} | p{1.1cm}<{\centering} | p{1.1cm}<{\centering} | p{1.2cm}<{\centering} | p{1.1cm}<{\centering} | p{1.1cm}<{\centering} | p{1.1cm}<{\centering}  | p{1.4cm}<{\centering} | p{1.3cm}<{\centering} | p{1.3cm}<{\centering} |}
				\hline
				\multirow{2}{1.2cm}{\centering Corruption Level } & \multicolumn{9}{c|}{Accuracy (\%)}\\	
				\cline{2-10}
				\multirow{2}{0.6cm}{}
				& NMF  			& GNMF				& RMNMF		 		& WNMF				& CNMF 				& ONMF				&	Semi-NMF		 & LS-NMF           & RLS-NMF \\ \hline			
				20\%    & 36.36& 36.36 & 38.18 & 38.18	& 34.55	& 30.30	& \textcolor{green}{39.39}	& \textcolor{blue}{44.85}    &\textcolor{red}{46.67}\\			
				40\%	& 31.52 & \textcolor{green}{39.39} & 38.18 & 38.79	& 27.27	& 34.55 & 27.88	& \textcolor{blue}{40.00} & \textcolor{red}{43.64}	\\
				60\%	& 30.91 &\textcolor{green}{32.73} & 30.91 & 31.52	& 26.06	& \textcolor{green}{32.73}	& 31.52	& \textcolor{blue}{39.39}  &\textcolor{red}{40.00}  \\
				\hline
				\hline
				\multirow{2}{1.2cm}{\centering Corruption Level }& \multicolumn{9}{c|}{Normalized Mutual Information (\%)}\\	
				\cline{2-10}
				\multirow{2}{0.6cm}{}
				& NMF  			& GNMF				& RMNMF		 		& WNMF				& CNMF 				& ONMF				&	Semi-NMF		 & LS-NMF           & RLS-NMF \\ \hline			
				20\%  	&39.55 & 42.21 & 41.38 & 43.43	& 39.71	& 36.98	& \textcolor{green}{46.60}	& \textcolor{blue}{46.62}  & \textcolor{red}{49.15}    \\			
				40\%	& 35.96 & 43.03 & 40.84 & \textcolor{green}{44.11}	& 33.55	& 38.08	& 34.39	& \textcolor{blue}{44.33}     &\textcolor{red}{44.62}	\\
				60\%	& 35.51 & \textcolor{green}{35.90} & 34.48 & 35.39	& 32.03	& 34.99	& 35.22	& \textcolor{blue}{42.59} & \textcolor{red}{44.09} \\
				\hline\hline
				\multirow{2}{1.2cm}{\centering Corruption Level }& \multicolumn{9}{c|}{Purity (\%)}\\	
				\cline{2-10}
				\multirow{2}{0.6cm}{}
				& NMF  			& GNMF				& RMNMF		 		& WNMF				& CNMF 				& ONMF				&	Semi-NMF		 & LS-NMF           & RLS-NMF \\ \hline
				20\%    & 37.58 & 38.79 & 40.00 & 40.00	& 37.58	& 32.73	& \textcolor{green}{43.64}	& \textcolor{blue}{44.85}  & \textcolor{red}{48.48}   	\\			
				40\%	& 34.55 & \textcolor{blue}{40.61} & 38.18 & \textcolor{green}{40.00}	& 30.91	& 35.15	& 29.09	& \textcolor{green}{40.00}	   &\textcolor{red}{43.64}	\\
				60\%	& 33.33 & 32.73 &\textcolor{green}{34.55} & 32.73 & 29.70	&\textcolor{green}{34.55}	& 33.94	&\textcolor{blue}{40.00}    &\textcolor{red}{41.21}	\\ \hline
				
			\end{tabular}
		}
		\label{tab_per_noisy_yale}
	\end{table*}
	
	\begin{table*}[t] 
		\scriptsize
		\centering
		\caption{Clustering Performance on Corrupted COIL20 Data Set}
		\resizebox{0.95\textwidth}{0.18\textwidth}{
			\begin{tabular}{|p{1.6cm}<{\centering} | p{1.1cm}<{\centering} | p{1.1cm}<{\centering} | p{1.2cm}<{\centering} | p{1.1cm}<{\centering} | p{1.1cm}<{\centering} | p{1.1cm}<{\centering}  | p{1.4cm}<{\centering} | p{1.3cm}<{\centering} | p{1.3cm}<{\centering} |}
				\hline
				\multirow{2}{1.2cm}{\centering Corruption Level }& \multicolumn{9}{c|}{Accuracy (\%)}\\	
				\cline{2-10}
				\multirow{2}{0.6cm}{}
				& NMF  & GNMF	& RMNMF	& WNMF	& CNMF	& ONMF	&	Semi-NMF & LS-NMF  & RLS-NMF \\ \hline		
				20\%    &60.76 &\textcolor{green}{79.31} &56.32 & 64.79 & 53.75	& 60.49	& 59.51 &\textcolor{blue}{80.97}  & \textcolor{red}{83.82}\\
				40\%	&56.53 &\textcolor{green}{67.85} &58.89 & 56.53 & 51.32	&58.47	& 56.94& \textcolor{blue}{76.25}	 &\textcolor{red}{80.49}	\\
				60\%	&61.46 & \textcolor{blue}{67.99} &56.74 & \textcolor{green}{61.60}	& 54.48	& 56.39	& 61.25	&\textcolor{red}{71.32}  &\textcolor{red}{71.32}  \\
				\hline
				\hline
				\multirow{2}{1.2cm}{\centering Corruption Level }& \multicolumn{9}{c|}{Normalized Mutual Information (\%)}\\	
				\cline{2-10}
				\multirow{2}{0.6cm}{}
				& NMF & GNMF	& RMNMF		& WNMF	& CNMF 	& ONMF	&Semi-NMF	& LS-NMF   & RLS-NMF \\ \hline			
				20\%  	&70.79 &\textcolor{green}{85.60} &68.89 & 72.60	& 66.88	& 71.60	& 70.91	&\textcolor{blue}{87.13} &\textcolor{red}{88.79}    \\
				40\%	& 69.36 &\textcolor{green}{74.60} & 69.48 & 69.36	& 63.75	& 68.83	& 69.60 & \textcolor{blue}{79.99} & \textcolor{red}{83.14}	\\
				60\%	& 69.32 & \textcolor{blue}{75.54} & 67.75 & \textcolor{green}{71.19}	& 62.87	& 66.77	& 68.12 & \textcolor{red}{77.30} & \textcolor{red}{77.30} \\
				\hline
				\multirow{2}{1.2cm}{\centering Corruption Level }& \multicolumn{9}{c|}{Purity (\%)}\\	
				\cline{2-10}
				\multirow{2}{0.6cm}{}
				& NMF  	& GNMF	& RMNMF		& WNMF & CNMF	& ONMF &	Semi-NMF & LS-NMF  & RLS-NMF \\ \hline
				20\%   & 61.46	& \textcolor{green}{80.76}	&57.99	&65.90	&58.82	&66.11	&61.53	&\textcolor{blue}{82.78}	&\textcolor{red}{84.10} \\
				40\%	&58.75	&\textcolor{blue}{71.46}	&61.18	&58.75	&56.04	&62.78	&57.92	&\textcolor{green}{76.32}	&\textcolor{red}{80.49} \\
				60\%	& 62.78	&\textcolor{blue}{69.03}	&60.35	&\textcolor{green}{62.92}	&56.87	&58.54	&61.81	&\textcolor{red}{73.26}	&\textcolor{red}{73.26} 	\\ \hline
				
			\end{tabular}
		}
		\\The top three performances are highlighted in {\color{red}red}, {\color{blue}blue}, and {\color{green}green}, respectively.
		\label{tab_per_noisy_COIL20}
	\end{table*}

	\subsection{Comparison of Clustering Performance}
	\label{sec_exp_clean}
	
	In this test, we conduct extensive experiments to testify the effectiveness of the proposed method.
	For all the methods, we follow the settings as described above and report the detailed clustering performance in \cref{tab_all}.
	It is observed that the proposed method generally achieves the best performance with significant improvements.
	In particular, the RLS-NMF achieves the best performances with significant improvements on 7 out of 8 data sets in clustering accuracy and purity, and 6 out of 8 data sets in NMI, respectively.
	On EYaleB and ORL data sets, the RLS-NMF improves the performance by about 10\% in accuracy and purity.
	In the other cases, the RLS-NMF achieves the top second performance with comparable performance to the best method.
	For example, on COIL20 data set, the RLS-NMF has the top performance in accuracy and purity, respectively.
	In NMI, the RSL-NMF achieves the top second performance, which is slightly inferior to the GNMF by 0.27\%.
	It is noted that among all baseline methods, GNMF and Semi-NMF are among the most competing ones.
	In particular, Semi-NMF has the best performance on AR data set in all metrics while GNMF has the best performance in NMI on COIL20 data set.
	In all other cases, these two methods have inferior performances to the RLS-NMF. 
	Generally, we observe that the RLS-NMF has better performance than the LS-NMF.
	However, on some data sets, such as Jaffe and COIL20, the RLS-NMF has the same performance as the LS-NMF.
	We explain this as follows: In such data sets as Jaffe and COIL20, the noise effects are not strong, which can be observed from \cref{fig_imgs}.
	Thus, we believe that the noise term and the $\ell_{2,\log}$-(pseudo) norm is not essential in this case.
	However, on other data sets with heavy noise effects, such as EYaleB data set, we can see that the RLS-NMF has significantly improved performance than the LS-NMF, which verifies the effectiveness and necessity of the robust model. 
	In the next test, we will further testify the RLS-NMF on data sets with artificial noise to confirm its effectiveness.

	\begin{figure}[!tb]
		\centering{
			\includegraphics[width=0.9\columnwidth]{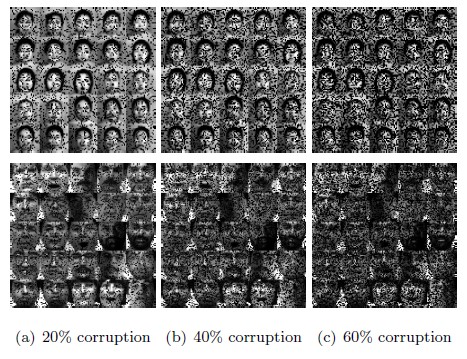}
		\caption{ Examples of the corrupted images from Jaffe (on top) and Yale (on bottom) data sets. From left to right are images with 20\%, 40\%, and 60\% level corruptions, respectively. }
		\label{fig_noisy_img}
		}
	\end{figure}

	\subsection{Comparison on Noisy Data with Random Corruptions}
	\label{sec_exp_corruption}
	{
		To further testify the effectiveness and robustness to noise effects of the RLS-NMF model,
		we conduct experiments on noisy data.
		In particular, we consider two types of noise in our experiments, including random corruption and Gaussian noise. 
		In this subsection, we first conduct experiments on randomly corrupted data.
		Throughout this subsection, we keep the same settings for all methods as in \cref{sec_exp_clean}. 
	}
	Without loss of generality, we conduct experiments on Yale, Jaffe, and COIL20 data sets, where we randomly remove 20\%, 40\%, and 60\% elements from these data sets, respectively. 
	We show some examples of the corrupted data examples in \cref{fig_noisy_img}. 
	It is seen that the images are severely damaged with the 60\% level of corruption, which makes the clustering of such data more challenging.
	We report the comparison results in \cref{tab_per_noisy_yale,tab_per_noisy_jaffe,tab_per_noisy_COIL20}.
	It is observed that the RLS-NMF obtains the best performances in all cases with significant improvements over the baseline methods.
	It should be noted that on original Yale and Jaffe data sets, the RLS-NMF obtains the same performance with the LS-NMF.
	However, with heavy noise effects, the RLS-NMF shows superior performance to the LS-NMF. 
	These observations show the improved robustness of the robust model on randomly corrupted data sets.

	%
	%
	%
	%
	%
	%
	%
	
	\begin{table*}[t] 
		\scriptsize
		\centering
		\caption{Clustering Performance on ORL Data Set with Gaussian Noise}
		\resizebox{0.95\textwidth}{0.18\textwidth}{
			\begin{tabular}{|p{1.6cm}<{\centering} | p{1.1cm}<{\centering} | p{1.1cm}<{\centering} | p{1.2cm}<{\centering} | p{1.1cm}<{\centering} | p{1.1cm}<{\centering} | p{1.1cm}<{\centering}  | p{1.4cm}<{\centering} | p{1.3cm}<{\centering} | p{1.3cm}<{\centering} |}
				\hline
				\multirow{2}{1.2cm}{\centering Variance }& \multicolumn{9}{c|}{Accuracy (\%)}\\	
				\cline{2-10}
				\multirow{2}{0.6cm}{}
				& NMF  & GNMF	& RMNMF	& WNMF	& CNMF	& ONMF	&	Semi-NMF & LS-NMF  & RLS-NMF \\ \hline		
				0.005    &56.50 &56.00 &\textcolor{green}{57.75} & 55.50 & 18.00	& 39.50	& 47.75 &\textcolor{blue}{61.00}  & \textcolor{red}{63.50}\\
				0.010	&49.75 &\textcolor{green}{52.75} &52.16 & 49.75 & 15.75	&28.50	& 32.75& \textcolor{blue}{54.25}	 &\textcolor{red}{54.75}	\\
				0.015	&33.50 & \textcolor{green}{39.25} &37.96 & 33.50	& 17.25	& 23.00	& 25.50	&\textcolor{blue}{40.75}  &\textcolor{red}{44.00}  \\
				\hline
				\hline
				\multirow{2}{1.2cm}{\centering Variance }& \multicolumn{9}{c|}{Normalized Mutual Information (\%)}\\	
				\cline{2-10}
				\multirow{2}{0.6cm}{}
				& NMF & GNMF	& RMNMF		& WNMF	& CNMF 	& ONMF	&Semi-NMF	& LS-NMF   & RLS-NMF \\ \hline			
				0.005  	&73.61 &73.86 &62.18 & \textcolor{green}{73.99}	& 37.39	& 59.44	& 65.01	&\textcolor{blue}{75.86} &\textcolor{red}{76.52}    \\
				0.010	& 67.32 &\textcolor{green}{69.67} & 69.10 & 67.32	& 34.09	& 49.39	& 51.49 & \textcolor{blue}{69.97} & \textcolor{red}{73.17}	\\
				0.015	& 56.62 & \textcolor{blue}{59.26} & 53.42 & 56.62	& 34.68	& 43.61	& 44.53 & \textcolor{green}{58.25} & \textcolor{red}{61.45} \\
				\hline
				\multirow{2}{1.2cm}{\centering Variance }& \multicolumn{9}{c|}{Purity (\%)}\\	
				\cline{2-10}
				\multirow{2}{0.6cm}{}
				& NMF  	& GNMF	& RMNMF		& WNMF & CNMF	& ONMF &	Semi-NMF & LS-NMF  & RLS-NMF \\ \hline
				0.005   & 60.75	& 60.00	&57.76	&\textcolor{green}{61.00}	&19.00	&43.50	&52.00	&\textcolor{blue}{64.50}	&\textcolor{red}{66.00} \\
				0.010	&54.75	&\textcolor{blue}{58.05}	&55.83	&54.75	&17.00	&32.25	&36.25	&\textcolor{green}{57.75}	&\textcolor{red}{60.25} \\
				0.015	& 37.75	&\textcolor{green}{43.25}	&42.85	&37.75	&17.50	&25.00	&28.00	&\textcolor{blue}{44.75}	&\textcolor{red}{48.25} 	\\ \hline
				
			\end{tabular}
		}
		\label{tab_per_gaussian_ORL}
	\end{table*}

	\begin{table*}[t] 
		\scriptsize
		\centering
		\caption{Clustering Performance on Yale Data Set with Gaussian Noise}
		\resizebox{0.95\textwidth}{0.18\textwidth}{
			\begin{tabular}{|p{1.6cm}<{\centering} | p{1.1cm}<{\centering} | p{1.1cm}<{\centering} | p{1.2cm}<{\centering} | p{1.1cm}<{\centering} | p{1.1cm}<{\centering} | p{1.1cm}<{\centering}  | p{1.4cm}<{\centering} | p{1.3cm}<{\centering} | p{1.3cm}<{\centering} |}
				\hline
				\multirow{2}{1.6cm}{\centering Variance }& \multicolumn{9}{c|}{Accuracy (\%)}\\	
				\cline{2-10}
				\multirow{2}{0.6cm}{}
				& NMF  & GNMF	& RMNMF	& WNMF	& CNMF	& ONMF	&	Semi-NMF & LS-NMF  & RLS-NMF \\ \hline		
				0.005    &40.61	&43.03	&\textcolor{green}{43.89}	&41.82	&24.85	&38.79	&41.82 &\textcolor{blue}{48.48}  & \textcolor{red}{49.70}\\
				0.010	&40.00	&\textcolor{blue}{44.24}	&39.77	&40.00	&32.12	&35.15	&36.36 & \textcolor{green}{43.03}	 &\textcolor{red}{48.48}	\\
				0.015	&36.36	&\textcolor{green}{41.21}	&37.56	&40.00	&23.64	&32.73	&36.36	&\textcolor{blue}{44.85}  &\textcolor{red}{46.06}  \\
				\hline
				\hline
				\multirow{2}{1.6cm}{\centering Variance }& \multicolumn{9}{c|}{Normalized Mutual Information (\%)}\\	
				\cline{2-10}
				\multirow{2}{0.6cm}{}
				& NMF & GNMF	& RMNMF		& WNMF	& CNMF 	& ONMF	&Semi-NMF	& LS-NMF   & RLS-NMF \\ \hline			
				0.005  	&45.68 &45.03	&44.66	&45.87	&32.26	&43.55	&\textcolor{green}{47.74} &\textcolor{red}{50.59} &\textcolor{blue}{49.04}    \\
				0.010	& 41.76	&\textcolor{blue}{49.16}	&41.06	&41.76	&37.92	&38.28	&\textcolor{green}{45.41} & 44.32 & \textcolor{red}{50.92}	\\
				0.015	&40.02	&\textcolor{green}{43.34}	&39.98	&42.47	&26.36	&34.47	&42.42 & \textcolor{red}{49.71} & \textcolor{blue}{49.03} \\
				\hline
				\multirow{2}{1.6cm}{\centering Variance }& \multicolumn{9}{c|}{Purity (\%)}\\	
				\cline{2-10}
				\multirow{2}{0.6cm}{}
				& NMF  	& GNMF	& RMNMF		& WNMF & CNMF	& ONMF &	Semi-NMF & LS-NMF  & RLS-NMF \\ \hline
				0.005   &41.21	&45.45	&43.89	&43.03	&29.70	&40.00	&\textcolor{green}{46.06}	&\textcolor{blue}{49.09}	&\textcolor{red}{49.70} \\
				0.010	&41.82	&\textcolor{blue}{45.45}	&39.97	&41.82	&33.33	&35.76	&38.79	&\textcolor{green}{44.24}	&\textcolor{red}{49.09} \\
				0.015	& 38.79	&\textcolor{green}{43.03}	&37.56	&40.61	&24.24	&33.33	&36.97	&\textcolor{blue}{46.67}	&\textcolor{red}{49.09} 	\\ \hline
				
			\end{tabular}
		}
		\label{tab_per_gaussian_Yale}
	\end{table*}
	
	\begin{table*}[t] 
		\scriptsize
		\centering
		\caption{Clustering Performance on COIL20 Data Set with Gaussian Noise}
		\resizebox{0.95\textwidth}{0.18\textwidth}{
			\begin{tabular}{|p{1.6cm}<{\centering} | p{1.1cm}<{\centering} | p{1.1cm}<{\centering} | p{1.2cm}<{\centering} | p{1.1cm}<{\centering} | p{1.1cm}<{\centering} | p{1.1cm}<{\centering}  | p{1.4cm}<{\centering} | p{1.3cm}<{\centering} | p{1.3cm}<{\centering} |}
				\hline
				\multirow{2}{1.6cm}{\centering Variance }& \multicolumn{9}{c|}{Accuracy (\%)}\\	
				\cline{2-10}
				\multirow{2}{0.6cm}{}
				& NMF  & GNMF	& RMNMF	& WNMF	& CNMF	& ONMF	&	Semi-NMF & LS-NMF  & RLS-NMF \\ \hline		
				0.005    &58.75	&\textcolor{green}{79.17}	&77.35	&63.40	&47.29	&59.65	&60.28
				&\textcolor{blue}{83.54}  & \textcolor{red}{83.96}\\
				0.010	&64.24	&\textcolor{green}{77.64}	&73.33	&64.24	&26.94	&61.32	&65.59
				& \textcolor{blue}{80.76}	 &\textcolor{red}{82.01}	\\
				0.015	&67.64	&\textcolor{green}{72.22}	&70.09	&59.93	&19.79	&57.50	&61.32
				&\textcolor{blue}{74.24}  &\textcolor{red}{77.50}  \\
				\hline
				\hline
				\multirow{2}{1.6cm}{\centering Variance }& \multicolumn{9}{c|}{Normalized Mutual Information (\%)}\\	
				\cline{2-10}
				\multirow{2}{0.6cm}{}
				& NMF & GNMF	& RMNMF		& WNMF	& CNMF 	& ONMF	&Semi-NMF	& LS-NMF   & RLS-NMF \\ \hline			
				0.005  	&73.76	&\textcolor{green}{85.85}	&85.07	&74.30	&56.95	&70.90	&70.91
				&\textcolor{blue}{90.35} &\textcolor{red}{90.63}    \\
				0.010	& 76.03	&84.98	&\textcolor{green}{85.59}	&76.03	&38.89	&70.72	&72.21
				& \textcolor{blue}{87.39} & \textcolor{red}{88.43}	\\
				0.015	& 75.95	&\textcolor{green}{79.03}	&77.74	&72.72	&28.28	&69.76	&68.40
				& \textcolor{blue}{79.47} & \textcolor{red}{80.98} \\
				\hline
				\multirow{2}{1.6cm}{\centering Variance }& \multicolumn{9}{c|}{Purity (\%)}\\	
				\cline{2-10}
				\multirow{2}{0.6cm}{}
				& NMF  	& GNMF	& RMNMF		& WNMF & CNMF	& ONMF &	Semi-NMF & LS-NMF  & RLS-NMF \\ \hline
				0.005   & 62.78	&\textcolor{green}{80.03}	&79.17	&65.35	&50.07	&62.64	&62.99
				&\textcolor{blue}{86.39}	&\textcolor{red}{86.45} \\
				0.010	&65.63	&\textcolor{green}{80.00}	&75.56	&65.63	&30.90	&62.78	&66.87
				&\textcolor{blue}{80.97}	&\textcolor{red}{82.15} \\
				0.015	& 69.10	&\textcolor{blue}{75.00}	&73.21	&62.64	&22.36	&59.24	&61.60
				&\textcolor{green}{74.31}	&\textcolor{red}{78.54} 	\\ \hline
				
			\end{tabular}
		}
		\label{tab_per_gaussian_COIL20}
	\end{table*}

	\begin{figure}[!tb]
		\centering{
			\includegraphics[width=0.9\columnwidth]{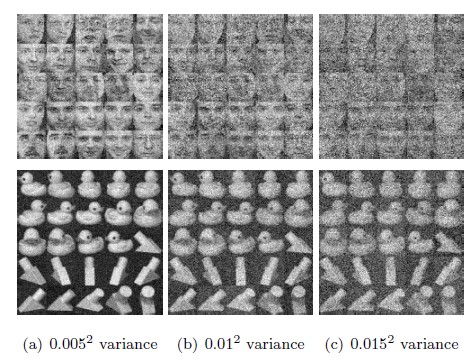}
			\caption{ Examples of the images with Gaussian noise from ORL (on top) and COIL20 (on bottom) data sets. From left to right are images with Gaussian noise with variance being 0.005$^2$, 0.01$^2$, and 0.015$^2$, respectively. }
			\label{fig_noisy_img_gaussian}
		}
	\end{figure}

    \begin{figure}[!tb]
    	\centering{
    	\includegraphics[width=0.9\columnwidth]{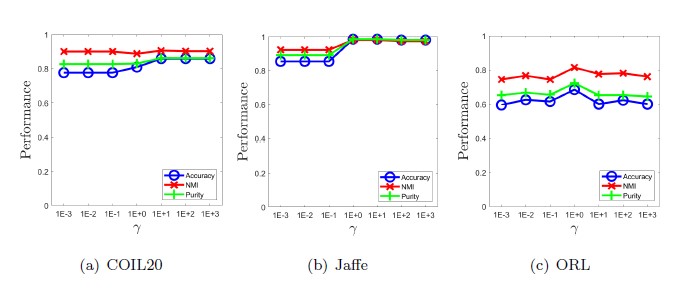}
    	\caption{Performance changes of RLS-NMF with respect to $\gamma$, where $\alpha$ and $\beta$ are fixed to be the optimal ones used in \cref{sec_exp_clean}, respectively. }
    	\label{fig_gamma}
    	}
    \end{figure}   
	
	{
		\subsection{Comparison on Noisy Data with Gaussian Noise }
		\label{sec_exp_gaussian}
		In this subsection, we further evaluate the proposed method on data sets with Gaussian noise. 
		In particular, we test all methods under different noise level conditions. 
		Without loss of generality, we conduct experiments using Yale, ORL, COIL20, and Semeion data sets.
		In our test, we add zero mean Gaussian noise to the data sets with variance varies within $\{0.005^2, 0.01^2, 0.015^2\}$, 
		which are referred as light, moderate, and heavy noise conditions in this test. 
		For each data set, we map the data to ensure nonnegativity by subtracting the smallest value. 
		We show some examples of ORL and COIL20 data sets for illustration of the noise effects in \cref{fig_noisy_img_gaussian}. 
		The other experimental settings remain the same as those in \cref{sec_exp_clean,sec_exp_corruption}. 
		We report the clustering results in \cref{tab_per_gaussian_Yale,tab_per_gaussian_COIL20,tab_per_gaussian_ORL}.
		As the noise becomes heavier, it is observed that all methods have significantly reduced performances, which confirms the adverse effects of noise. 
		Generally, we can see that the LS-NMF and RLS-NMF obtain the best performances among all methods. 
		Moreover, the RLS-NMF has relatively better performance than the LS-NMF.
		For example, the RLS-NMF has slightly improved performance over LS-NMF on COIL20 data set under the light noise condition. 
		However, the RLS-NMF has significantly improved performance over LS-NMF on COIL20 data set under moderate and heavy noise conditions,
		which, again, confirms the enhanced robustness of the RLS-NMF model and the effectiveness of using the $\ell_{2,\log}$-norm.

	}

	 \begin{figure}[!tb]
		\centering{
			\includegraphics[width=0.9\columnwidth]{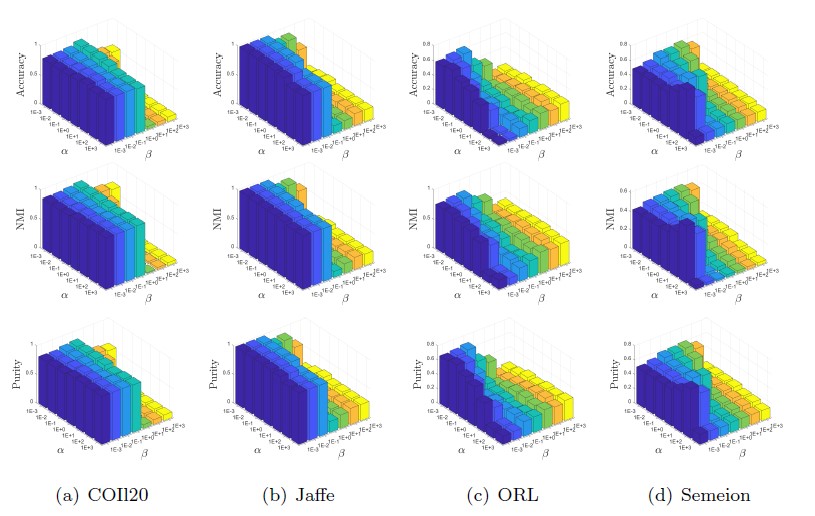}
			\caption{ Performance changes of RLS-NMF with respect to $\alpha$ and $\beta$ on different data sets, where $\gamma$ is fixed to be the optimal one used in \cref{sec_exp_clean}. }
		\label{fig_paras}
		}
	\end{figure}   
	\subsection{Parameter Sensitivity}
	For unsupervised learning methods, it is difficult to determine optimal parameters in real-world applications.
	Thus, it is important that unsupervised method performs well insensitively to parameters.
	{
		In this test, we show the sensitivity of RLS-NMF to the balancing parameters. 
		Without loss of generality, we show the results on four data sets, including COIL20, Jaffe, ORL, and Semeion.
		On other data sets, we can observe similar patterns.
		Specifically, we conduct experiments from two perspectives.
		That is, we first show the effects of parameter $\gamma$ and then show how the combination of parameters $\{\alpha,\beta\}$ affects the final clustering performance of RLS-NMF, respectively.}
	For $\gamma$, we fix $\alpha$ and $\beta$ to be the ones used in \cref{sec_exp_clean} and vary $\gamma$ within the set $\{0.001,0.01,0.1,1,10,100,1000\}$.
	We show the results in \cref{fig_gamma}.
	It is observed that with a broad range of values for $\gamma$, the RLS-NMF can achieve good performance.
	We can also observe that the RLS-NMF tends to perform better with larger $\gamma$ values, 
	which might be explained that larger $\gamma$ values render the RLS-NMF better account for noise effects.

	For $\{\alpha,\beta\}$, we fix $\gamma$ to be the ones used in \cref{sec_exp_clean} and vary $\alpha$ and $\beta$ within the set $\{0.001,0.01,0.1,$ $1,10,100,1000\}$.
	We show the results in \cref{fig_paras}.
	{
		It is seen that the RLS-NMF obtains relatively high performance within a broad range of parameter selections.}
	In particular, the RLS-NMF is insensitive to $\alpha$ and small $\beta$ values tend to be more effective.
	We observe similar patterns on other data sets, which suggests us to set small values for $\beta$ in real-world applications.

	%
	%
	%

	 \begin{figure}[!tb]
		\centering{
			\includegraphics[width=0.9\columnwidth]{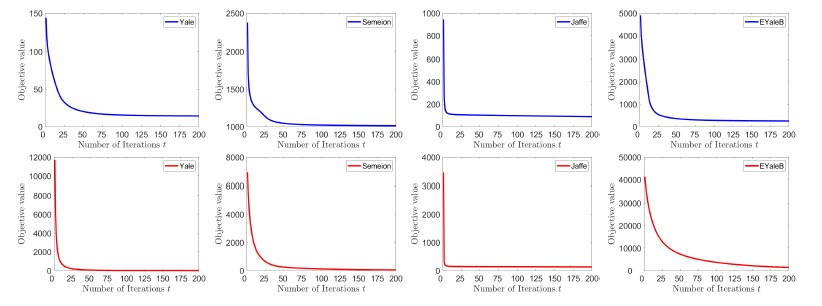}
		\caption{Examples of convergence curves of the LS-NMF (on top) and RLS-NMF (on bottom) on different data sets. }
		\label{fig_conv}
		}
	\end{figure} 
	\subsection{Convergence Analysis and Time Comparison}
	In \cref{sec_conv_sl,sec_conv_rsl}, we have provided theoretical analysis of the convergence of the objective value of the LS-NMF and RLS-NMF methods.
	In this section, we further provide experimental results to verify the convergent property of the optimization algorithms.
	{
		Without loss of generality, we conduct experiments on four data sets, including Yale, Semeion, Jaffe, and EYaleB,
		where we plot how the value of objective function changes with respect to the iteration numbers for both LS-NMF and RLS-NMF.
		For each data set, we fix the parameters to be the ones used in \cref{sec_exp_clean}, which lead to the highest clustering performance.
		We show the curves of the first 200 iterations in \cref{fig_conv}.}
	It is observed that the objective value sequences on these data sets indeed converge, which empirically verifies the convergent property of the proposed algorithms.
	Moreover, we observe that the proposed algorithms generally converge within about 100 iterations, which implies their fast convergence and efficiency.

	 \begin{figure}[!tb]
		\centering{
			\includegraphics[width=0.9\columnwidth]{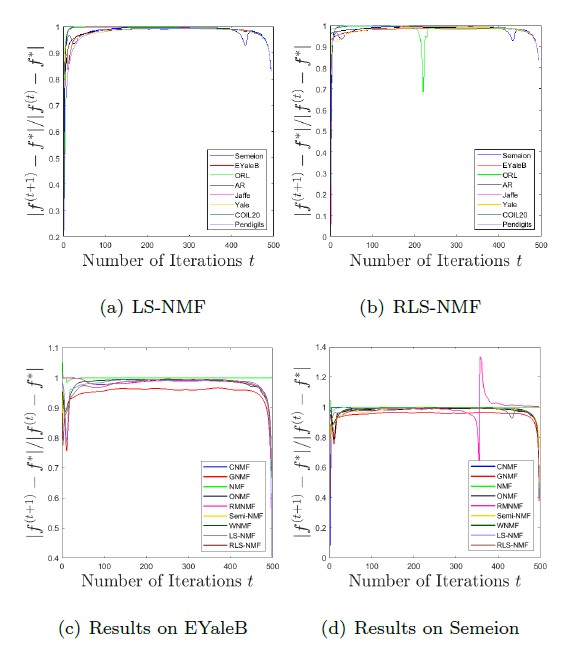}
			\caption{Empirical results on convergence rates of different methods: (a)-(b) curves of LS-NMF and RLS-NMF on all data sets, respectively; (c)-(d) comparison of all methods on EYaleB and Semeion data sets, respectively. }
			\label{fig_convrate}
		}
	\end{figure} 
	
	{
		Then, we further empirically investigate the convergence rate of the proposed algorithms. 
		In particular, we plot the values of $\big\{\frac{|f^{(t+1)} - f^*|}{|f^{(t)} - f^*|}\big\}$ in our experiment,
		where $\{f^{(t)}\}$ denotes the objective value sequence and $f^*$ denotes the convergent value of $\{f^{(t)}\}$. 
		As seen in above test, the proposed LS-NMF and RLS-NMF algorithms converge within about 200 iterations.
		Thus, in this test we treat $f^{(500)}$ as the empirical value of $f^*$, 
		where 500 is sufficiently large to guarantee the convergence of $\{f^{(t)}\}$, and show the results in \cref{fig_convrate_LS,fig_convrate_RLS}.
		It is seen that, as far as can be observed, the values of $\frac{|f^{(t+1)} - f^*|}{|f^{(t)} - f^*|}$ are always less than 1, 
		which indicates that the sequence $\{f^{(t)}\}$ is convergent with a linear convergence rate.
		Without loss of generality, we test all algorithms on EYaleB and Semeion data sets to compare their convergence rates and show the results in \cref{fig_convrate_EYaleB,fig_convrate_Semeion}.
		It is seen that the curves generally imply that these methods have a linear convergence rate,
		which basically suggests that all the methods have comparable convergence rates.

		Moreover, we test the time cost for all methods and show the results in \cref{fig_time}, where all algorithms are terminated after 500 iterations. 
		This test is conducted with MATLAB 2018a on a workstation with Intel(R) Xeon(R) W-2133 CPU. 
		It is seen that the LS-NMF and RLS-NMF have at least comparable efficiency to other methods, such as ONMF.
		Considering that the LS-NMF and RLS-NMF have superior performance to the other baseline methods in clustering,
		such comparable speed, though not the fastest, is indeed acceptable. 
		
	 \begin{figure}[!tb]
		\centering{
			\includegraphics[width=0.9\columnwidth]{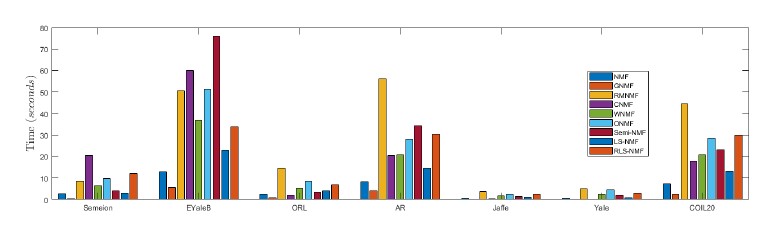}
				\caption{Comparison of time cost of all methods on different data sets. }
			\label{fig_time}
		}
	\end{figure}

	}

     \begin{figure}[!tb]
    	\centering{
    		\includegraphics[width=0.9\columnwidth]{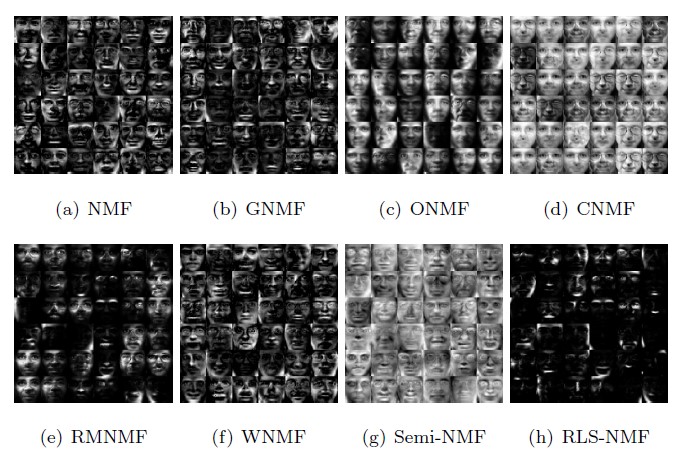}
    		\caption{ Visual examples of the basis vectors learned by various methods from ORL data set by different methods. 
    		}
    		\label{fig_basis}
    	}
    \end{figure} 	
	%
	

	\subsection{Sparsity Learning}
	\label{sec_basis}
	{
		In this test, we investigate the ability of the proposed method in sparsity learning with the $\ell_{\log}$-(pseudo) norm regularization.
		Throughout this test, we fix the parameters to those used in \cref{sec_exp_clean} for all methods,
		where these parameters lead to the best clustering performance.
		As pointed out in \cite{cai2011graph}, sparser basis learned by NMF models implies better parts-based representation.
		In this test, we first show some examples of the basis learned by different methods.}
	Without loss of generality, we conduct this test on ORL data set.
	The learned basis vectors have 1024 dimensions and we reshape them to size of 32$\times$32 for visual illustration.
	Then we visually show the reshaped basis vectors as gray scale images in \cref{fig_basis}.
	It is observed that the basis images learned by Semi-NMF are not sparse, which is explained by the mixed signs of the basis.
	Among all methods, RLS-NMF generates the sparsest basis vectors as can be observed from the results.
	The sparser basis vectors suggest that RLS-NMF can learn better parts-based representation than the baseline methods, 
	which implies its effectiveness in finding a low-dimensional representation of the data.
	
	{
	
	    \begin{figure}[!tb]
	    	\centering{
	    		\includegraphics[width=0.9\columnwidth]{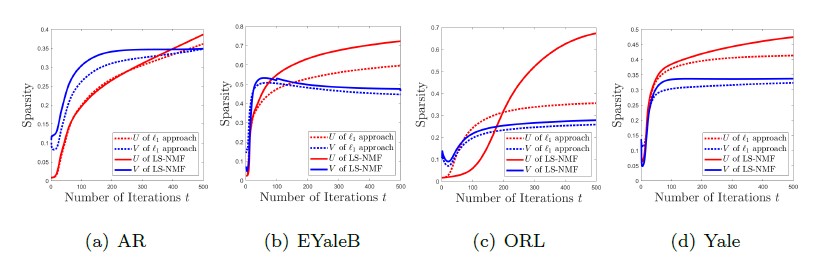}
	    		\caption{ Comparison of LS-NMF and the $\ell_1$ approach in sparse learning. }
	    		\label{fig_sparseness}
	    	}
	    \end{figure} 	
	}
	
	{
		
		Moreover, to verify the effectiveness of the $\ell_{\log}$-(pseudo) norm in restricting sparseness of the factor matrices,
		we compare the LS-NMF with the $\ell_1$-norm approach.
		The $\ell_1$-norm approach is obtained by replacing the $\ell_{\log}$-(pseudo) norm in \cref{eq_obj_log2} with the $\ell_1$-norm,
		where multiplicative updating rules are used in a way similar to LS-NMF.
		For both approaches, we show the results on several data sets, including AR, EYaleB, Yale, and ORL, where the parameters are tuned such that best clustering performance is observed.
		We show how the sparsity of $U$ and $V$ matrices changes with respect to the iterations in \cref{fig_sparseness},
		where the sparsity follows the definition given in \cite{hoyer2004non}.
		It is seen that the LS-NMF has superior performance in restricting the sparsity of the basis and representation vectors,
		which verifies the effectiveness of the proposed approach in learning parts-based representation.
	}

	

	\begin{figure}[!tb]
		\centering{
			\includegraphics[width=0.9\columnwidth]{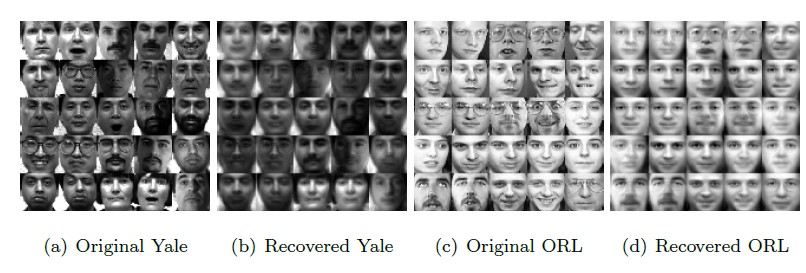}
			\caption{ Visual examples of the original and reconstructed images. (a)-(b) are from Yale while (c)-(d) are from ORL, respectively. }
			\label{fig_reconstruction}
		}
	\end{figure} 	
	
	\subsection{Data Reconstruction}
	To better understand the factorization results of the proposed method, we further show some examples of reconstructed data by RLS-NMF.
	Without loss of generality, we show some reconstructed example images from Yale and ORL data sets in \cref{fig_reconstruction}.
	The parameters are set to be the ones used in \cref{sec_exp_clean}, which leads to the best clustering performance.
	It is observed that the reconstructed images well capture the key features of original images.
	Moreover, some outliers like glasses are significantly removed in the reconstructed images. 
	These observations verify the effectiveness of the RLS-NMF in finding parts-based representations.

	\section{Conclusion}
	\label{sec_conclusion}
	In this paper, we propose a new type of sparse NMF methods, including the LS-NMF and RLS-NMF, with the latter being the robust version of the former. 
	The RLS-NMF learns the basis and representation matrices with $\ell_{\log}$-(pseudo) norm, which enhances the sparseness of the learned parts-based representation. 
	Moreover, to enhance the robustness of the RLS-NMF model to noise effects, a noise term is introduced, which is restricted to be example-wisely sparse with the novel $\ell_{2,\log}$-(pseudo) norm.
	Efficient multiplicative updating rules are developed for optimization, which have theoretical convergence guarantee.
	Extensive experiments verify the effectiveness of the RLS-NMF in clustering and data representation. 
	
	\section*{Acknowledgment}
	This work is supported by National Natural Foundation of China (NSFC) under Grants 61806106, 62172246, 61802215, and 61806045; 
	Shandong Provincial Natural Science Foundation, China under Grants ZR2019QF009, and ZR2019BF011; 
	Open Project Program of State Key Laboratory of Virtual Reality Technology and Systems, Beihang University, under Grant VRLAB2021A05; 
	Q.C. is partially supported by NIH R21AG070909 and UH3 NS100606-03 and a grant from the University of Kentucky.

	\bibliographystyle{model2-names}
	\bibliography{logNMF}
	
\end{document}